\documentclass{article}

\usepackage{subfigure}
\usepackage{booktabs} 

\usepackage{hyperref}

\usepackage[accepted]{arxiv}

\usepackage[utf8]{inputenc} 
\usepackage[T1]{fontenc}    
\usepackage{hyperref}       
\usepackage{url}            
\usepackage{booktabs}       
\usepackage{amsfonts}       
\usepackage{nicefrac}       
\usepackage{microtype}      
\usepackage{lipsum}
\usepackage{amssymb}
\usepackage{fancyhdr}       
\usepackage{graphicx}       
\usepackage{xcolor}
\usepackage{enumerate}
\usepackage{textcomp}
\usepackage{mathrsfs}
\usepackage{amsthm}
\usepackage{bbm}
\usepackage{epigraph}
\usepackage{etoolbox}
\usepackage{aligned-overset}
\usepackage{thmtools, thm-restate}
\usepackage{proof-at-the-end}

\usepackage[capitalize,noabbrev]{cleveref}

\theoremstyle{plain}
\newtheorem{theorem}{Theorem}[section]

\newtheorem{lemma}[theorem]{Lemma}
\newtheorem{corollary}[theorem]{Corollary}
\theoremstyle{definition}
\newtheorem{definition}[theorem]{Definition}

\theoremstyle{remark}
\newtheorem{remark}[theorem]{Remark}

\usepackage[textsize=tiny]{todonotes}

\def\KL{\mathbf{d}_{\mathrm{KL}}}

\def\normal{\mathcal{N}}

\def\diffentropy{\bf h}

\def\E{\mathbb{E}}
\def\F{\mixture}
\def\H{\mathbb{H}}
\def\diffentropy{\mathbf{h}}
\def\I{\mathbb{I}}
\def\Pr{\mathbb{P}}

\def\L{\mathbb{L}}
\def\1{\mathbf{1}}

\def\mixture{\mathbb{F}}

\newcommand{\Xc}{\mathcal{X}}
\newcommand{\Yc}{\mathcal{Y}}

\newcommand{\Lc}{\mathcal{L}}

\newcommand{\Z}{\mathbb{Z}}
\usepackage{hyperref}

\usepackage[mathscr]{euscript}
\hypersetup{
    colorlinks=true,
    linkcolor=blue,
    filecolor=magenta,      
    urlcolor=purple,
}
\usepackage{xcolor}

\newcount\Comments
\newcommand{\kibitz}[2]{\ifnum\Comments=1{\textcolor{#1}{\textsf{\footnotesize #2}}}\fi}
\usepackage{color}
\definecolor{darkred}{rgb}{0.7,0,0}
\definecolor{darkgreen}{rgb}{0.0,0.5,0.0}
\definecolor{darkblue}{rgb}{0.0,0.0,0.5}
\definecolor{teal}{rgb}{0.0,0.5,0.5}

\arxivtitlerunning{An Information-Theoretic Analysis of ICL}
\ifdefined\usebigfont

\usepackage{times}
\usepackage[fontsize=13pt]{scrextend}
\makeatletter
\@ifpackageloaded{geometry}{\AtBeginDocument{\newgeometry{letterpaper,left=1.56in,right=1.56in,top=1.71in,bottom=1.77in}}}{\usepackage[letterpaper,left=1.56in,right=1.56in,top=1.71in,bottom=1.77in]{geometry}}
\AtBeginDocument{\newgeometry{letterpaper,left=1.56in,right=1.56in,top=1.71in,bottom=1.77in}}
\else
\fi

\begin{document}

\ifdefined\usebigfont

\onecolumn
\else
\twocolumn[
\fi
\arxivtitle{An Information-Theoretic Analysis of \\
           In-Context Learning}



\arxivsetsymbol{equal}{*}

\begin{arxivauthorlist}
\arxivauthor{Hong Jun Jeon}{stanfordcs}
\arxivauthor{Jason D. Lee}{princeton}
\arxivauthor{Qi Lei}{nyu}
\arxivauthor{Benjamin Van Roy}{stanford}
\end{arxivauthorlist}

\arxivaffiliation{princeton}{Princeton University, Princeton, NJ, USA}
\arxivaffiliation{stanfordcs}{Department of Computer Science, Stanford University, Stanford, CA, USA}
\arxivaffiliation{stanford}{Stanford University, Stanford, CA, USA}
\arxivaffiliation{nyu}{New York University, New York City, NY, USA}

\arxivcorrespondingauthor{Hong Jun Jeon}{hjjeon@stanford.edu}

\arxivkeywords{Meta-Learning, Sequential Learning, Information Theory}

\vskip 0.3in
]



\printAffiliationsAndNotice{}  

\begin{abstract}
Previous theoretical results pertaining to meta-learning on sequences build on contrived assumptions and are somewhat convoluted.
We introduce new information-theoretic tools that lead to an elegant and very general decomposition of error into three components: irreducible error, meta-learning error, and intra-task error.  These tools unify analyses across many meta-learning challenges.  To illustrate, we apply them to establish new results about in-context learning with transformers.  Our theoretical results characterizes how error decays in both the number of training sequences and sequence lengths.  Our results are very general; for example, they avoid contrived mixing time assumptions made by all prior results that establish decay of error with sequence length.
\end{abstract}

\section{Introduction}
In recent years, we have observed the capability of large language models (LLMs) to learn from data within just its context window.  This puzzling phenomenon referred to as in-context learning (ICL) \citep{brown2020language}, has captured the attention of the theoretical machine learning community.  As the data available in-context is dwarfed by the extensive pretraining set, meta-learning stands as a prevailing explanation for ICL \cite{xie2022explanation}.

As aforementioned, \citet{xie2022explanation} introduced the idea that ICL could be interpreted as implicit Bayesian inference within a mixture of HMMs.  While their theoretical results rely on contrived assumptions and fail to explain how ICL is possible with such short sequences, their work \emph{initiated} the study of modeling ICL as Bayesian inference or other thoroughly studied learning processes such as empirical risk minimization.  As much of the theoretical community is most familiar with error analyses of empirical risk minimization, much of the existing results \citep{pmlr-v202-li23l, bai2023transformers, DBLP:journals/corr/abs-2110-10090} study the error of an ICL under the assumption that ICL is competitive in out-of-sample performance with empirical risk minimization.  However, each of these error bounds is limited in some way such as exponential depth dependence \citep{DBLP:journals/corr/abs-2110-10090, pmlr-v202-li23l} or error which decays only with the number of sequences and not the length of the sequences \citep{DBLP:journals/corr/abs-2110-10090, bai2023transformers}.  The results which do demonstrate that error decays in both the number of training sequences and sequence length often rely on contrived mixing time assumptions \citep{zhang2023does} or stability conditions which are equivalent to fast mixing \citep{pmlr-v202-li23l}.

Our work revisits the idea of modeling ICL as Bayesian inference.  In this work, we introduce new information-theoretic tools based on work by \citet{jeon2023informationtheoretic} which lead to an elegant and very general decomposition of error in meta-learning from sequences.  This decomposition consists of three components: irreducible error, meta-learning error, and intra-task error.  This unifies theoretical error analyses across many meta-learning challenges.  Notably, our results provide an error bound which decays linearly in both the number of sequences and the lengths of the sequences without explicit reliance on any stability or mixing assumptions within the sequence.  To demonstrate the use of our results, we specialize our theory to reproduce existing results in linear representation learning and to produce new results pertaining to a sparse mixture of transformer models.  The latter result provides a compelling narrative as to how ICL is possible with such few examples.

As some of our tools are non-standard to much of the community, we begin by introducing our framework in the simpler setting of learning from a single sequence of data.  In the following section, we naturally extend the analysis to meta-learning from many sequences and present our main result (Theorem \ref{th:main_result}).  Since our results are very general and abstract, we demonstrate the application of these results to several concrete problem instances.  In the main text, we provide concrete examples which resemble learning from data generated by a deep transformer model and in the appendix we provide simpler problem instances for reference (logistic regression, linear representation learning).  

\section{Related Works}
\paragraph{In-context Learning and Transformer.}
LLMs based on the transformer architecture \citep{vaswani2023attention} have exhibited the ability to learn from data within the context of a prompt \citep{brown2020language}.  This phenomenon, referred to as in-context learning (ICL), has received significant empirical investigation ~\citep{liu2021makes,min2021metaicl,lu2021fantastically,zhao2021calibrate,rubin2021learning,elhage2021mathematical,kirsch2022general,wei2023larger,brown2020language,dong2022survey}. 

However, theoretical understanding of ICL is still relatively nascent ~\citep{xie2022explanation,garg2022can,von2023transformers,dai2022can,giannou2023looped,pmlr-v202-li23l,raventos2023effects}.  Among the existing theoretical work, most focuses on the optimization dynamics~\citep{tian2023scan,tian2023joma,jelassi2022vision,li2023transformers,tarzanagh2023transformers,zhang2023trained,huang2023context,ahn2023transformers,mahankali2023one} or the representation power~\citep{sanford2023representational,song2023uncovering,von2023transformers,giannou2023looped,liu2022transformers} regarding the transformer architecture.  In the realm of statistical results, much of the existing work is confined to how transformers can perform ICL by simulating gradient descent~\citep{von2023transformers,akyurek2022learning,dai2022can,giannou2023looped}.  However, as they provide no concrete sample complexity results, they are therefore not directly comparable to our work.  The work that is perhaps most relevant to ours include those which analyze the sample complexity of ICL under the assumption that its performance is comparable to empirical risk minimization or Bayesian inference ~\cite{xie2022explanation,pmlr-v202-li23l,bai2023transformers,DBLP:journals/corr/abs-2110-10090,zhang2023does}.  Despite their quantitative sample complexity results, as mentioned in the introduction, these results are ultimately limited by either their restrictive assumptions on mixing times of the data sequence or their inability to capture how sequence length contributes to reduction in error.

\paragraph{Meta-learning.} 
As our work analyzes ICL under the lens of meta-learning, we provide a brief exposition of its existing work.  Recent empirical advancements have sparked interest in the theoretical foundations of meta-learning~\citep{baxter2000model,denevi2018incremental,finn2019online}.  In settings such as tasks drawn from a shared meta-distribution, several works \citep{maurer2009transfer, pontil2013excess, maurer2016benefit} have derived generalization bounds albeit for simplistic settings such as linear representation or linear classifiers. Under strong assumptions such as large margin or large number of tasks~\citet{srebro2006learning,aliakbarpour2023metalearning} were also able to establish such bounds.  However, these results all rely on the assumption that the data \emph{within} each meta-task is independently and identically distributed (iid) under an (unknown) probability distribution.  However, in the context of LLMs, for which the meta-tasks are separate documents, the sequence of tokens within each document is certainly not iid.  Our work provides novel theoretical tools which facilitate the analysis of meta-learning from sequential data which may not be iid.


\section{Learning from Sequential Data}

For exposition, we begin by introducing our general information-theoretic tools for the analysis of standard \emph{supervised learning} on sequential data.  Examples of such learning problems include but are not limited to natural language modeling and learning from video/audio data.  Phenomena such as ICL in LLMs is another fascinating instance of machine learning from sequential data.  Results from this section draw inspiration from \citep{jeon2023informationtheoretic} which focused on the analysis of supervised learning from \emph{iid} data.

We model all uncertain quantities as random variables.  Each random variable we consider is defined with respect to a common probability space $(\Omega, \F, \Pr)$.  Of particular interest to our analysis is a sequence $X_1, X_2, \ldots, X_T$ of discrete random variables which represent observations.  This sequence is generated by an autoregressive model parameterized by a random variable $\theta$ such that for all $t\in \Z_{+}$, $X_{t+1}$ may depend on $\theta$ and the entire history $X_1, \ldots, X_{t}$, which we abbreviate as $H_t$.



\subsection{Bayesian Error}

Our framework is {\it Bayesian} in the sense that it treats learning as the process of reducing uncertainty about $\theta$, which is taken to be a random variable.  A learning algorithm produces, for each $t$, a \emph{predictive distribution} $P_t$ of $X_{t+1}$ after observing the history $H_t$.  We express such an algorithm in terms of a function $\pi$ for which $P_t = \pi(H_t)$.  For a horizon $T\in\Z_{++}$, we quantify the error realized by predictions $P_t$ for $t < T$ in terms of the average cumulative expected log-loss:
$$\L_{T,\pi} = \frac{1}{T}\sum_{t=0}^{T-1}\ \E_\pi\left[-\ln P_t(X_{t+1})\right].$$

\subsection{Achievable Bayesian Error}

A natural question is: which $\pi$ minimizes the Bayesian error? The following result establishes that across all problem instances, the optimal algorithm $\pi$ sets $P_t = \Pr(X_{t+1}\in\cdot|H_t)$ for all $t$. We denote this \emph{Bayesian posterior} by $\hat{P}_t$.
\begin{lemma}{\bf (Bayesian posterior is optimal)}\label{le:bayes_optimal_seq}
    For all $t\in\Z_{+}$,
    $$ \E\left[-\ln\hat{P}_t(X_{t+1})\big|H_t\right]\ \overset{a.s.}{=}\ \min_{\pi}\ \E_\pi \left[-\ln P_t(X_{t+1})|H_t\right].$$
\end{lemma}
\begin{proof}
    In the below proof take all equality to hold \emph{almost surely}.
    \begin{align*}
        &\ \E\left[-\ln P_t(X_{t+1})|H_t\right]\\
        & = \E\left[-\ln \hat{P}_t(X_{t+1}) + \ln\frac{\hat{P}_t(X_{t+1})}{P_t(X_{t+1})}\Big| H_t\right]\\
        & = \E\left[-\ln \hat{P}_t(X_{t+1})\Big| H_t\right] +\KL(\hat{P}_t\|P_t).
    \end{align*}
    The result follows from the fact that $\KL(\hat{P}_{t}\|P_{t}) > 0$ for all $P_{t} \neq \hat{P}_{t}$.
\end{proof}

We use $\L_T$ to denote the \emph{optimal} achievable Bayesian error:
\begin{align*}
    \L_T 
    & = \frac{1}{T}\sum_{t=0}^{T-1}\ \E\left[- \ln\hat{P}_t(X_{t+1})\right].
\end{align*}
In the main text we restrict our attention to the study of \emph{optimal} achievable Bayesian error but we provide an extension to arbitrary predictors which depend on the history $H_t$ in Appendix \ref{apdx:suboptimal}.  The following result provides an exact characterization of the optimal cumulated expected log-loss.
\begin{restatable}{theorem}{seqBayesError}{\bf (Bayesian error)}\label{th:BayesError}
    For all $T \in \Z_{+}$,
    $$\L_{T} = \underbrace{\frac{\H(H_T|\theta)}{T}}_{\substack{\text{irreducible}\\ \text{error}}} + \underbrace{\frac{\I(H_{T};\theta)}{T}}_{\substack{\text{estimation}\\ \text{error}}}.$$
\end{restatable}
\begin{proof}
    \begin{align*}
        \L_T
        & = \frac{1}{T}\sum_{t=0}^{T-1}\ \E\left[-\ln\hat{P}_j(X_{t+1})\right]\\
        & = \frac{1}{T}\sum_{t=0}^{T-1}\ \E\left[\ln\frac{1}{\Pr(X_{t+1}|H_t,\theta)} + \ln\frac{\Pr(X_{t+1}|H_t,\theta)}{\hat{P}_t(X_{t+1})}\right]\\
        & = \frac{1}{T}\sum_{t=0}^{T-1}\ \H(X_{t+1}|\theta, H_t)\\
        &\  + \frac{1}{T}\sum_{t=0}^{T-1} \E\left[\KL(\Pr(X_{t+1}\in\cdot|H_t,\theta)\|\hat{P}_t(X_{t+1}\in\cdot))\right]\\
        & \overset{(a)}{=} \frac{\H(H_T|\theta)}{T} + \frac{1}{T}\sum_{t=0}^{T-1} \I(X_{t+1};\theta|H_t)\\
        & \overset{(b)}{=} \frac{\H(H_T|\theta)}{T} + \frac{\I(H_T;\theta)}{T},
    \end{align*}
    where $(a)$ and $(b)$ follow from the chain rule of conditional mutual information.
\end{proof}

\citet{jeon2023informationtheoretic} establish Theorem \ref{th:BayesError} in the setting in which the sequence is iid when conditioned on $\theta$.  We refer to $\H(H_{T}|\theta)$ as the \emph{irreducible error} because it is the error incurred by even the \emph{omniscient} predictor $\Pr(X_{t+1}\in\cdot|\theta,H_t)$.  The \emph{estimation error} represents statistical error incurred by an agent that produces estimates of the future $X_{t+1}$ from the past sequence $H_t$.  Since estimation error encompasses error which is \emph{reducible} via learning, our analysis will focus on characterizing this quantity.  We use
\begin{align*}
    \Lc_T
    & = \frac{\I(H_T;\theta)}{T},
\end{align*}
to denote the estimation error.  $\Lc_T$ will often vanish as $n\to \infty$.  For instance, if $\H(\theta)<\infty$, then this will trivially be the case as $\I(H_t;\theta)\leq \H(\theta)$ for all $t$.  However, even in problems for which $\H(\theta) = \infty$, for example if $\theta$ is a continuous random variable, the estimation error will still often vanish as $n\to \infty$.  Note that $\H(\theta)$ should not be confused with $\diffentropy(\theta)$, the \emph{differential entropy} of $\theta$.  The differential entropy does not capture the same qualitative properties as discrete entropy, namely $1)$ invariance under change of variables, $2)$ non-negativity.  While \emph{differences} in differential entropy still provide meaningful insight via mutual information $(\I(X;Y) = \diffentropy(X)-\diffentropy(X|Y))$, the quantity itself is largely vacuous for the purposes of measuring information content and therefore deriving error bounds.  The appropriate extension of discrete entropy to continuous random variables can be made via \emph{rate-distortion theory}.

\begin{definition}{(\bf rate-distortion function)} Let $\epsilon\geq 0$, $\theta:\Omega\mapsto\Theta$ be a random variable, and $\rho$ a distortion function which maps $\theta$ and a random variable $\tilde{\theta}$ to $\Re$. The rate-distortion function evaluated for random variable $\theta$ at tolerance $\epsilon$ takes the value:
$$\inf_{\tilde{\theta}\in\tilde{\Theta}_\epsilon}\ \I(\theta;\tilde{\theta}),$$
where
$$\tilde{\Theta}_\epsilon = \left\{\tilde{\theta}: \rho(\theta,\tilde{\theta}) \leq \epsilon\right\}.$$
\end{definition}
One can think of $\tilde{\theta}$ as a lossy \emph{compression} of the random variable $\theta$. The objective $\I(\theta;\tilde{\theta})$, referred to as the \emph{rate}, characterizes the number of nats that $\tilde{\theta}$ retains about $\theta$. Meanwhile, the distortion function $\rho$ characterizes how lossy the compression is.  When we apply rate-distortion theory to the analysis of machine learning, we restrict our attention to the case in which
\begin{align*}
    & \rho(\theta,\tilde{\theta})\\
    & = \E\left[\KL(\Pr(X_{t+1}\in\cdot|\theta, H_t)\|\Pr(X_{t+1}\in\cdot|\tilde{\theta}, H_t))\right]\\
    & = \I(X_{t+1};\theta|\tilde{\theta}, H_t).
\end{align*}
We assume that $\tilde{\theta}\perp X_{t+1}|(\theta, H_t)$ (the compression $\tilde{\theta}$ does not contain exogenous information about $X_{t+1}$, such as aleatoric noise, which cannot be determined from $(\theta, H_{t})$). We use the notation $\H_{\epsilon, T}(\theta)$ to denote the rate-distortion function w.r.t. this KL-divergence distortion function averaged across horizon $T$:
$$\H_{\epsilon,T}(\theta) = \inf_{\tilde{\theta}\in\tilde{\Theta}_{\epsilon.T}}\ \I(\theta;\tilde{\theta}),$$
where
\begin{align*}
    \tilde{\Theta}_{\epsilon, T}
    & = \left\{\tilde{\theta}: \tilde{\theta}\perp H_T|\theta;\quad \frac{\I(H_T;\theta|\tilde{\theta})}{T} \leq \epsilon \right\}.
\end{align*}

With this notation established, we present the following result for sequential learning.  The proof can be found in Appendix \ref{apdx:seq}.

\begin{restatable}{theorem}{seqBayesRD}{\bf (rate-distortion estimation error bound)}\label{th:bayes_rd}
    For all $T \in \Z_{+}$,
    $$\sup_{\epsilon\geq 0}\ \min\left\{\frac{\H_{\epsilon, T}(\theta)}{T} , \epsilon \right\}\ \leq\ \Lc_T\ \leq\ \inf_{\epsilon\geq 0}\ \frac{\H_{\epsilon, T}(\theta)}{T} + \epsilon.$$
\end{restatable}

An interpretation of the above result is that the Bayesian posterior implicitly finds the compression $\tilde{\theta}$ that optimally trades off learning complexity $\I(\theta;\tilde{\theta})$ and distortion $\I(H_T;\theta|\tilde{\theta})$.  While these results are very general, they remain abstract.  In Appendix \ref{apdx:log_reg} we provide a simple logistic regression example.  In the main text, we provide an analysis for learning from a sequence generated by a deep transformer model.

\subsection{Deep Transformer}\label{subsec:transformer}

In the transformer environment, we let $(X_1, X_2, \ldots)$ be a sequence in $\{1, \ldots, d\}$, where $d$ denotes the size of the vocabulary.  Each of the $d$ outcomes is associated with a \emph{known} embedding vector which we denote as $\Phi_j$ for $j \in \{1,\ldots, d\}$.  We assume that for all $j$, $\|\Phi_j\|_2 = 1$.  For brevity of notation, we let $\phi_t = \Phi_{X_t}$ i.e. the embedding associated with token $X_t$.

Let $K$ denote the context length of the transformer, $L$ denote it's depth, and $r$ denote the attention dimension.  
We assume that the first token $X_1$ is sampled from an arbitrary pmf on $\{1, \ldots, d\}$ but subsequent tokens are sampled based on the previous $K$ tokens within the context window and the weights of a depth $L$ transformer model.

We use $U_{t,i}$ to denote the output of layer $i$ at time $t$ $(U_{t,0} = \phi_{t-K+1:t})$ (the embeddings associated with the past $K$ tokens).  For all $t \leq T, i < L$, let

$$\text{Attn}_i(U_{t,i-1}) = \sigma\left( \frac{U^{\top}_{t,i-1} A_iU_{t,i-1}}{\sqrt{r}}\right)$$
denote the attention matrix of layer $i$ where $\sigma$ denotes the softmax function applied elementwise along the columns.  The matrix $A_i \in \Re^{r\times r}$ can be interpreted as the product of the key and query matrices and without loss of generality, we assume that the elements of the matrices $A_i$ are distributed iid $\normal(0,1)$ (Gaussian assumption is not crucial but known mean and unit variance is).

Subsequently, we let
$$U_{t,i} = \text{Clip}\left(V_i U_{t,i-1} \text{Attn}_i(U_{t,i-1})\right),$$
where Clip ensures that each column of the matrix input has $L2$ norm at most $1$.  The matrix $V_i$ resembles the value matrix and without loss of generality, we assume that the elements of $V_i$ are distributed iid $\normal(0,1/d)$ (same generality conditions as above).

Finally, the next token is generated via sampling from the softmax of the final layer:
$$X_{t+1} \sim \sigma\left(U_{t,L}[-1]\right),$$
where $U_{t,L}[-1]$ denotes the right-most column of $U_{t,L}$.  At each layer $i$, the parameters $\theta_i$ consist of the matrices $A_i, V_i$.  We will use the notation $\theta_{i:j}$ for $i \leq j$ to denote the collection $(\theta_i, \theta_{i+1},\ldots, \theta_j)$.

\begin{restatable}{theorem}{tsfmRd}{\bf (transformer estimation error bound)}\label{th:tsfmRd}
    For all $d,r,L,K$, if $\theta_{1:L}$ is the transformer environment, then
    \begin{align*}
        \Lc_{T}
        & \leq \frac{(d^2+r^2)L^2\log(4K^2)}{T} + \frac{(d^2+r^2)L\log\left(\frac{2KT^2}{L}\right)}{2T}.\\
    \end{align*}
\end{restatable}

We note that even if the sequence generated by the transformer is not iid, we observe that $\Lc_{T}$ decays linearly in $T$, the length of the sequence.  Furthermore, we observe that $\Lc_{T}$ is upper bounded linearly in the product of parameter count and depth of the transformer model as in \citet{bai2023transformers}.  In the following section, we will draw the connection to ICL by studying meta-learning in a data generating process which resembles a sparse \emph{mixture} of deep transformers.

\section{Meta-Learning from Sequential Data}
In this section, we analyze the achievable performance of \emph{meta-learning} from sequences.  The tools of the Bayesian framework apply exactly as they do in standard supervised learning from sequences.  An example of meta-learning from sequences includes language model pretraining in which each ``meta-task'' can be interpreted as a separate document and the ``sequence'' as the tokens which comprise the document.  We will use the terminology \emph{document} going forward to refer to a ``meta-task'' in meta-learning.

\subsection{Data Generating Process}\label{subsec:datagen}


We now consider sequential data which resembles a \emph{corpus} of text documents.  We assume that all documents in the corpus have an identical length wich we denote by $T$.  For each document $m$, we let $D_m = X^{(m)}_1, \ldots, X^{(m)}_T$ be the sequence of discrete random variables which resembles its constituent tokens.

Each document is associated with a random variable $\theta_m$ which encodes information that is specific to document $m$.  As in the previous section, we assume that the sequence $D_m$ is produced by an autoregressive process.  As such, for all $t$, the value of $X^{(m)}_{t+1}$ depends on $\theta_m$ and the prior tokens $(X^{(m)}_1, \ldots, X^{(m)}_t)$ in $D_m$.

Finally, we assume that there exists a random variable $\psi$ such that conditioned on $\psi$, $(\theta_1, \theta_2, \ldots)$ is an iid sequence.  Note that $\psi$ encodes information which learnable \emph{across} documents in a corpus.  As such, $\psi$ represent the \emph{meta} parameters while $(\theta_1, \theta_2, \ldots)$ represent the \emph{intra-task} parameters.  Two natural conditional independence results follow from our formulation.  $1)$ for all $m$, $D_m \perp \psi|\theta_m$; the meta parameters do not contain information about $D_m$ beyond what is contained in $\theta_m$.  $2)$ $X^{(m)}_{t} \perp X^{(n)}_{t}|\psi$ for all $m \neq n$; tokens \emph{across} documents do not contain information about each other beyond what is contained in $\psi$.

\subsection{Bayesian Error}
Our framework is {\it Bayesian} in the sense that it treats learning as the process of reducing uncertainty about $\theta_1, \ldots, \theta_m, \psi$, which are taken to be random variables.  For a meta-learning problem with $M$ documents each of length $T$, a learning algorithm produces, for each $(m,t) \in [M]\times[T]$, a \emph{predictive distribution} $P_{m,t}$ of $X^{(m)}_{t+1}$ after observing the concatenated history which we denote by
$$H_{m,t} = (D_1, D_2, \ldots, D_{m-1}, X^{(m)}_1, \ldots, X^{(m)}_t).$$
$H_{m,t}$ consists of \emph{all} tokens from documents $1, \ldots, m-1$ and up to the $t$th token of document $m$.  We express our meta-learning algorithm in terms of a function $\pi$ for which $P_{m,t} = \pi(H_{m,t})$.  For all $M, T\in\Z_{++}$, we quantify the error realized by predictions $P_{m,t}$ for $(m,t) \in [M]\times[T]$ in terms of the average cumulative expected log-loss:
$$\L_{M,T,\pi} = \frac{1}{MT}\sum_{m=1}^{M}\sum_{t=0}^{T-1}\ \E_\pi\left[-\ln P_{m,t}\left(X^{(m)}_{t+1}\right)\right].$$
We note that this objective largely resembles the objective LLMs minimize in the process of pre-training.

\subsection{Achievable Bayesian Error}

We are in particular interested in the algorithm $\pi$ which minimizes Bayesian error. Just as in supervised learning from sequences, across all problem instances, the optimal algorithm $\pi$ sets $P_{m,t} = \Pr(X^{(m)}_{t+1}\in\cdot|H_{m,t})$ for all $m,t$. We denote this \emph{Bayesian posterior} by $\hat{P}_{m,t}$.

\begin{restatable}{lemma}{metaBayesOpt}{\bf (Bayesian posterior is optimal)}\label{le:meta_bayes_optimal_seq}
    For all $m,t\in\Z_{+}$,
    \begin{align*}
        &\ \E\left[-\ln\hat{P}_{m,t}\left(X^{(m)}_{t+1}\right)\big|H_{m,t}\right]\\
        & \overset{a.s.}{=}\ \min_{\pi}\ \E_\pi \left[-\ln P_{m,t}\left(X_{t+1}^{(m)}\right)|H_{m,t}\right].
    \end{align*}
\end{restatable}

We use $\L_{M,T}$ to denote the \emph{optimal} achievable Bayesian error:
\begin{align*}
    \L_{M,T}
    & = \frac{1}{MT}\sum_{m=1}^{M}\sum_{t=0}^{T-1}\ \E\left[- \ln\hat{P}_{m,t}\left(X^{(m)}_{t+1}\right)\right].
\end{align*}

We will restrict our attention to the performance of the optimal predictor $\hat{P}_t$.  We now present the main result of this paper which decomposes optimal Bayesian error into 3 intuitive terms.  The following result provides an exact characterization of $\L_{M,T}$.
\begin{restatable}{theorem}{metaSeqBayesError}{\bf (Main Result)}\label{th:main_result}
    For all $M,T \in \Z_{+}$ and $m \in \{1, 2, \ldots, M\}$,
    \begin{align*}
        \L_{M,T}
        & = \underbrace{\frac{\H(H_{M,T}|\theta_{1:M})}{MT}}_{\substack{\text{irreducible}\\ \text{error}}} + \underbrace{\frac{\I(H_{M,T};\psi)}{MT}}_{\substack{\text{meta}\\ \text{estimation}\\ \text{error}}}\\
        &\quad + \underbrace{\frac{\I(D_m;\theta_m|\psi)}{T}}_{\substack{\text{intra-document}\\ \text{estimation}\\ \text{error}}}.
    \end{align*}
\end{restatable}
\begin{proof}
    \begin{align*}
        & \L_{M,T}\\
        & = \frac{1}{MT}\sum_{m=1}^{M}\sum_{t=0}^{T-1}\ \E\left[- \ln\hat{P}_{m,t}(X^{(m)}_{t+1})\right]\\
        & \overset{(a)}{=} \frac{1}{MT}\sum_{m=1}^{M}\sum_{t=0}^{T-1}\ \H(X_{t+1}^{(m)}|\theta_m, H_{m,t})\\
        &\quad + \frac{1}{MT}\sum_{m=1}^{M}\sum_{t=0}^{T-1}\ \I(X_{t+1}^{(m)}; \psi, \theta_m|H_{m,t})\\
        & \overset{(b)}{=} \frac{1}{MT}\sum_{m=1}^{M}\H(D_m|\theta_m,H_{m-1,T})\\
        &\quad + \frac{1}{MT}\sum_{m=1}^{M}\I(D_m;\psi, \theta_m|H_{m-1,T})\\
        & \overset{(c)}{=} \frac{\H(H_{M,T}|\theta_{1:M})}{MT} +\frac{1}{MT}\sum_{m=1}^{M}\I(D_m;\psi|H_{m-1,T})\\
        &\quad +\frac{1}{MT}\sum_{m=1}^{M}\I(D_m;\theta_m|\psi, H_{m-1,T})\\
        & \overset{(d)}{=} \frac{\H(H_{M,T}|\theta_{1:M})}{MT} + \frac{\I(H_{M,T};\psi)}{MT} + \frac{\I(D_m;\theta_m|\psi)}{T},
    \end{align*}
    where $(a)$ follows from Theorem \ref{th:BayesError}, and $(b), (c), (d)$ follow from the chain rule of mutual information.
\end{proof}

The irreducible error represents the Bayesian error incurred by even the omniscent predictor $\Pr(X^{(m)}_{t+1}\in\cdot|\theta_m, H_{m,t})$ which conditions on document-specific information $\theta_m$ and the document history $H_{m,t}$.

The meta-estimation error represents the statistical error incurred in the process of estimating the meta parameters $\psi$.  Since all tokens across all documents contain information about $\psi$, it is intuitive that meta-estimation error term decays linearly in $MT$.  Since $M$ could in practice be very large (for example in a pretraining dataset), $\L_{M,T}$ could be small even for small $T$ if significant learning complexity is contained in $\psi$.

Finally, the intra-document estimation error represents the statistical error incurred in the process of learning $\theta_m$ after already conditioning on $\psi$.  As only the data from document $m$ $(D_m)$ pertains to $\theta_m$, this error intuitively decays linearly in $T$, the length of the document.  As mentioned before, if much of the learning complexity is contained in $\psi$, then $\I(H^{(1)}_T;\theta_m|\psi)$ will be small and therefore the intra-document estimation error may be small even for short document length $T$.  We will revisit this idea in section \ref{sec:in-context} when we analyze ICL within this framework.

Our subsequent analysis will focus on estimation error as it represents error which is reducible via learning.  In meta-learning, the total estimation error is:
$$\Lc_{M,T} = \frac{\I(H_{M,T};\psi)}{MT} + \frac{\I(D_m;\theta_m|\psi)}{T},$$
i.e. the sum of meta and intra-document estimation errors.

We note that Theorem \ref{th:main_result} holds for all data generating processes which meet the natural  assumptions made in subsection \ref{subsec:datagen}.  It is surprising that we can arrive at such a result which decays linearly in both $M$, the number of documents, and $T$, the lengths of the documents without any explicit reliance on stability or mixing assumptions. 

While the main result is useful for conceptual understanding, we need further tools to facilitate the theoretical analysis of concrete meta-learning problem instances.  To extend this result, we again use rate-distortion theory under the following modified rate-distortion functions:
\begin{align*}
    \H_{\epsilon,T}(\theta_m|\psi)
    & = \inf_{\tilde{\theta}_m\in\tilde{\Theta}_{\epsilon,T}}\ \I(\theta_m;\tilde{\theta}_m|\psi),
\end{align*}
where

\begin{align*}
    \tilde{\Theta}_{\epsilon, T}
    & = \left\{\tilde{\theta}: \tilde{\theta}\perp H_{M,T}|\theta_m;\quad \frac{\I(D_m;\theta_m|\tilde{\theta},\psi)}{T} \leq \epsilon \right\},
\end{align*}
and
$$\H_{\epsilon,M,T}(\psi) = \inf_{\tilde{\psi}\in\tilde{\Psi}_{\epsilon,M,T}}\ \I(\psi;\tilde{\psi}),$$
where
$$\tilde{\Psi}_{\epsilon, M, T} = \left\{\tilde{\psi}: \tilde{\psi}\perp H_{M,T}|\psi;\quad \frac{\I(H_{M,T};\psi|\tilde{\psi})}{MT} \leq \epsilon \right\}.$$

With this notation in place, we establish the following upper and lower bounds on $\Lc_{M,T}$ in terms of the above rate distortion functions.

\begin{restatable}{theorem}{metaRD}{\bf (rate-distortion estimation error bound)}\label{th:meta_bayes_rd}
    For all $M, T \in \Z_{+}$, and $m \in \{1, \ldots, M\}$,
    \begin{align*}
        \Lc_{M,T}
        & \leq \inf_{\epsilon\geq 0}\ \frac{\H_{\epsilon,M,T}(\psi)}{MT} + \epsilon + \inf_{\epsilon' \geq 0}\ \frac{\H_{\epsilon', T}(\theta_m|\psi)}{T} + \epsilon',
    \end{align*}
    and
    \begin{align*}
        \Lc_{M,T}
        & \geq \sup_{\epsilon\geq 0}\min\left\{\frac{\H_{\epsilon,M,T}(\psi)}{MT}, \epsilon\right\}\\
        &\quad + \sup_{\epsilon' \geq 0}\min\left\{ \frac{\H_{\epsilon', T}(\theta_m|\psi)}{T}, \epsilon'\right\}.
    \end{align*}
\end{restatable}

A direct consequence of Theorem \ref{th:meta_bayes_rd} is an upper bound on Bayesan error with respect to \emph{entropy} (by setting $\epsilon, \epsilon'$ to $0$).  While the utility of such a bound is limited to settings in which $\psi, \theta_{1:M}$ are discrete random variables, it may be useful to the reader conceptually.  The bound is captured in the following Corollary:
\begin{corollary}{\bf (entropy estimation error bound)}
    For all $M, T \in \Z_{+}$, and $m \in \{1,\ldots, M\}$
    $$\Lc_{M,T} \leq  \frac{\H(\psi)}{MT} + \frac{\H(\theta_m|\psi)}{T}.$$
\end{corollary}

In the following section, we will apply Theorem \ref{th:meta_bayes_rd} to derive error bounds for a sparse mixture of (deep) transformers.  For a simpler linear representation learning example, we refer the reader to Appendix \ref{apdx:lin_rep}.  

\subsection{Sparse Mixture of Transformers}
In the sparse mixture of transformers environment, for all documents $m$, we let its tokens $(X^{(m)}_1, X^{(m)}_2, \ldots)$ be a sequence in $\{1, \ldots, d\}$, where $d$ denotes the size of the vocabulary.  Each of the $d$ outcomes is associated with a \emph{known} embedding vector which we denote as $\Phi_j$ for $j \in \{1,\ldots, d\}$.  We assume that for all $j$, $\|\Phi_j\|_2 = 1$.  For brevity of notation, we let $\phi_t^{(m)} = \Phi_{X^{(m)}_t}$ i.e. the embedding associated with token $X^{(m)}_t$.

Each document is generated by a transformer model which is sampled iid from a mixture.  We assume that sampling is performed according to a categorical distribution parameterized by $\psi$ with prior distribution $\Pr(\psi\in\cdot) = \text{Dirichlet}(N, [R/N, \ldots, R/N])$ for a scale parameter $R \ll N$.  Under this prior distribution, the expected number of unique outcomes grows linearly in $R$ and only logarithmically in the number of draws ($M$ in our case).  As a result, we permit the size of the mixture $N$ to potentially be exponentially large, but we assume that the mixture's complexity is controlled by the sparsity parameter $R$.

Each of the $N$ elements of the mixture corresponds to a deep transformer network as outlined in Section \ref{subsec:transformer}.  Let $K$ denote the context lengths of the transformers, $L$ denote their depths, and $r$ their attention dimensions.  
We assume that for all documents, the first token $X^{(m)}_1$ is sampled from an arbitrary pmf on $\{1, \ldots, d\}$ but subsequent tokens are sampled based on the previous $K$ tokens within the context window and the weights of the sampled transformer model.

The tokens of each document are generated according to the weights of the sampled transformer and the previous $K$ tokens.  The generation of token $X^{(m)}_{t+1}$ will depend on $\theta_m$ and $X^{(m)}_{t-K+1},\ldots, X^{(m)}_{t}$.  For all $m, t$, we let $(U^{(m)}_{t,0} = \phi_{t-K+1:t})$ refer to the embeddings associated with the past $K$ tokens.  For $i > 0$, we let $U^{(m)}_{t,i}$ denote the output of layer $i$ of the transformer with input $U^{(m)}_{t,0}$.  For all $t \leq T, i < L, m \leq M$, let

$$\text{Attn}_i(U^{(m)}_{t,i-1}) = \sigma\left( \frac{U^{(m)\top}_{t,i-1} A^{(m)}_iU^{(m)}_{t,i-1}}{\sqrt{r}}\right)$$
denote the attention matrix of layer $i$ for document $m$ where $\sigma$ denotes the softmax function applied elementwise along the columns.  The matrix $A^{(m)}_i \in \Re^{r\times r}$ can be interpreted as the product of the key and query matrices and without loss of generality, we assume that the elements of the matrices $A^{(m)}_i$ are distributed iid $\normal(0,1)$ (Gaussian assumption is not crucial but known mean and unit variance is).

Subsequently, we let
$$U^{(m)}_{t,i} = \text{Clip}\left(V^{(m)}_i U^{(m)}_{t,i-1} \text{Attn}_i(U^{(m)}_{t,i-1})\right),$$
where Clip ensures that each column of the matrix input has $L2$ norm at most $1$.  The matrix $V^{(m)}_i$ resembles the value matrix and without loss of generality, we assume that the elements of $V^{(m)}_i$ are distributed iid $\normal(0,1/d)$ (same generality conditions as above).

Finally, the next token is generated via sampling from the softmax of the final layer:
$$X^{(m)}_{t+1} \sim \sigma\left(U^{(m)}_{t,L}[-1]\right),$$
where $U^{(m)}_{t,L}[-1]$ denotes the right-most column of $U^{(m)}_{t,L}$.  At each layer $i$, the parameters $\theta_{m,i}$ consist of the matrices $A^{(m)}_i, V^{(m)}_i$.



We provide the following novel result which upper bounds the error of the optimal Bayesian learner when learning from data generated by the sparse mixture of transformers.

\begin{restatable}{theorem}{metaTransformer}{\bf (mixture of transformers estimation error bound)}
    For all $d,r,K,L,M,T\in\Z_{++}$, if $\theta_1,\ldots, \theta_M, \psi$ are the sparse mixture of transformers environment and $r \leq d$, then
    \begin{align*}
        \Lc_{M,T}
        & \leq \frac{R\log\left(1+\frac{M}{R}\right)
        \log(MN)}{MT}\\
        &\ +\frac{R\log\left(1+\frac{M}{R}\right)(d^2+r^2)L^2\log\left(4K^2MT^2\right)}{MT}\\
        &\ + \frac{\log(N)}{T}.
    \end{align*}
\end{restatable}

We now provide some qualitative comments about this result.  The first and second terms denote the meta estimation error, and the third term denotes the intra-document estimation error.

The first term is the error incurred in the process of learning $\psi$, the probabilities by which the models of the mixture are sampled.  Note that even if there are $N$ models in the mixture, due to the Dirichlet assumption, the error depends linearly on $R$ the sparsity parameter and only logarithmically on $N$.  Note that this term decays linearly in $MT$ since data across documents provide information about $\psi$.

The second term measures the error incurred from learning the weights of the sampled models within the mixture.  Note that again, due to the Dirichlet assumption, this term scales only logarithmically in $M$.  This is because even if a model is resampled for every document, several documents may still be generated by the \emph{same} model from the mixture.  As a result, the dependence is linear in $R$ and only logarithmic in $M$. The remaining terms are linear in the product of parameter count and depth, which corroborates the results of \citet{bai2023transformers}.  However, our result decays linearly in $MT$ as opposed to just $M$ as in \citep{bai2023transformers}.  This is intuitive as the error ought to decrease in both the number of documents $M$ and the \emph{length} of the documents $T$.  This is an advantage of the Bayesian framework as it does not rely on a uniform convergence argument which requires mixing time assumption on the tokens within the document to obtain linear decay in $T$.

Finally, the third term is the intra-document estimation error which is the error incurred in the process of learning which model from the mixture generated each document.  Since there are $N$ different elements in the mixture, the $\log(N)/T$ is straightforward.  The longer the document length $T$, the more certain we should be about which model generated the document, hence lower error.  In the following section, we explicitly outline the connection between this example and ICL.

\subsection{In-context Learning as Meta-Learning from Sequences}\label{sec:in-context}
We now explicitly draw the connection between ICL and meta-learning from sequences.  We assume that the pretraining dataset consists of $M$ documents, each of length $T$.  We assume that a new $M+1$th document type is drawn and an in-context learner is described by an algorithm which produces for each $t$ a predictive distribution $P^{in}_{t}$ of $X_{t+1}^{(M+1)}$ after observing the history $H_{M+1, t}$ which consists of the pretraining data an the $t$ provided in the current context.  We let $D_{M+1} = (X^{(M+1)}_1, \ldots, X^{(M+1)}_\tau)$ denote the entire in-context sequence.  Note that we have summarize the effect of pretraining by allowing $P^{(in)}_t$ to depend on the pretraining history $H_{M,T}$.  We quantify error realized by predictions $P^{in}_t$ in terms of the average cumulative expected log-loss:

$$\L_{M,T,\tau,\pi} = \frac{1}{\tau}\sum_{t=0}^{\tau-1}\E_{\pi}\left[-\log P^{in}_{t}\left(X_{t+1}^{(M+1)}\right)\right],$$

where $\tau$ denotes the full length of the in-context sequence.  We assume that $\tau \leq T$ as $\tau$ can be at most $K$, the context-length of the transformer and the document lengths $T$ in pretraining are often much larger than $K$.  As before, we establish that $\hat{P}_t(X_{t+1}^{(M+1)}\in\cdot) = \Pr(X_{t+1}^{(M+1)}\in\cdot|H_{M+1, t})$ minimizes this loss almost surely.

\begin{theorem}
    For all $M,T,t \in \Z_{+}$,
    \begin{align*}
        & \E\left[-\log\hat{P}_t(X_{t+1}^{(M+1)})|H_{M+1,t}\right]\\
        & \overset{a.s.}{=} \min_{\pi}\ \E_{\pi}\left[\-\log P^{in}_t(X^{(M+1)}_{t+1})|H_{M+1,t}\right].
    \end{align*}
\end{theorem}

Going forward, we will restrict our attention to the performance of $\hat{P}_t$ which we denote as:
\begin{align*}
    \L_{M,T,\tau}
    & = \frac{1}{\tau} \sum_{t=0}^{\tau-1} \E\left[-\log \hat{P}_{t}(X_{t+1}^{(M+1)})\right].
\end{align*}

With this notation in place, we present an upper bound for the ICL error.  A proof can be found in Appendix \ref{apdx:icl}.

\begin{restatable}{theorem}{icl}{\bf (in context learning error bound)}
    For all $M,T,\tau \in \Z_{++}$, if $\tau \leq T$, then
    \begin{align*} 
        \L_{M,T,\tau}
        & \leq \underbrace{\frac{\H\left(D_{M+1}|\theta_{M+1}\right)}{\tau}}_{\substack{\text{irreducible}\\ \text{error}}} + \underbrace{\frac{\I(H_{M,T};\psi)}{M\tau}}_{\substack{\text{meta}\\ \text{estimation}\\ \text{error}}}\\
        &\quad + \underbrace{\frac{\I(D_{M+1};\theta_{M+1}|\psi)}{\tau}}_{\substack{\text{in-context}\\ \text{estimation}\\ \text{error}}}.
    \end{align*}
\end{restatable}

Note that if $M$ is large i.e. the number of pretraining documents is large, then almost all of the error will be attributed to the in-context estimation error:

\begin{remark}
    For sufficiently large $M$ (number of pretraining documents),
    $$\L_{M,T,\tau}\ \lesssim\ \underbrace{\frac{\H\left(D_{M+1}|\theta_{M+1}\right)}{\tau}}_{\substack{\text{irreducible}\\ \text{error}}}  + \frac{\I(D_{M+1};\theta_{M+1}|\psi)}{\tau}.$$
\end{remark}

\subsection{Discussion of Results}
If each pretraining document is generated by a transformer model which is drawn from a mixture as in the previous section, the above remark suggests for a sufficiently large pretraining set, the in-context error can be small for even modest values of $\tau$.  The in-context error is upper bounded by $\log(N)/\tau$ where $N$ is the size of the mixture.  Effectively, the in-context data only needs to distinguish which model from the mixture generated the current sequence.  As a result, the complexity is at most $\log(N)$ and the error decays linearly in the length of the in-context sequence $\tau$.  This corroborates work by \citet{min2022rethinking} which established that an in-context sequence largely augments performance via providing information about the distributions of the inputs and labels as well as the format of the sequence.  The LLMs is not literally learning from the examples, as even when the labels of examples were randomly scrambled, performance on downstream tasks was only marginally impacted.  This lends credence to the hypothesis that ICL pinpoints which model from the mixture is most suitable for the given in-context sequence.



\section{Conclusion}
In this work, we introduced novel information-theoretic tools to analyze the error of meta-learning from sequences.  Our tools produced very general and intuitive results which suggest that the error should decay in both the number of training sequences and the sequence lengths.  Notably, these results hold without relying on contrived mixing time assumptions as common in existing work.  By applying these tools, we developed novel results about ICL in transformers and a plausible mathematical hypothesis for how learning is possible even when only a small amount of data is provided in-context.  While the results of the main text are limited to exact Bayesian inference, we provide results in the Appendix which extend to \emph{suboptimal} algorithms as well.  A further rigorous investigation into the mechanisms by which transformers may be implementing a mixture of models would provide stronger credence to the hypothesis and results provided in this work.  

\bibliography{example_paper}
\bibliographystyle{arxiv}

\newpage
\appendix
\onecolumn


\section{Learning from Sequential Data}\label{apdx:seq}

\seqBayesError*
\begin{proof}
    \begin{align*}
        \L_T
        & = \frac{1}{T}\sum_{t=0}^{T-1}\ \E\left[- \ln\hat{P}_j(X_{t+1})\right]\\
        & = \frac{1}{T}\sum_{t=0}^{T-1}\ \E\left[-\ln\Pr(X_{t+1}|H_t,\theta) + \ln\frac{\Pr(X_{t+1}|H_t,\theta)}{\hat{P}_t(X_{t+1})}\right]\\
        & = \frac{1}{T}\sum_{t=0}^{T-1}\ \H(X_{t+1}|\theta, H_t)+ \E\left[\KL(\Pr(X_{t+1}\in\cdot|H_t,\theta)\|\hat{P}_t(X_{t+1}\in\cdot))\right]\\
        & \overset{(a)}{=} \frac{\H(H_T|\theta)}{T} + \frac{1}{T}\sum_{t=0}^{T-1} \I(X_{t+1};\theta|H_t)\\
        & \overset{(b)}{=} \frac{\H(H_T|\theta)}{T} + \frac{\I(H_T;\theta)}{T},
    \end{align*}
    where $(a)$ and $(b)$ follow from the chain rule of conditional mutual information.
\end{proof}

\seqBayesRD*
\begin{proof}
    \begin{align*}
        \Lc_{T}
        & = \frac{\I(H_T;\theta)}{T}\\
        & = \frac{1}{T}\sum_{t=0}^{T-1} \I(X_{t+1};\theta| H_t)\\
        & = \frac{1}{T}\sum_{t=0}^{T-1} \I(X_{t+1};\theta,\tilde{\theta}|H_t)\\
        & = \frac{1}{T}\sum_{t=0}^{T-1} \I(X_{t+1};\tilde{\theta}|H_t) + \I(X_{t+1};\theta|\tilde{\theta}, H_t)\\
        & = \frac{\I(H_T;\tilde{\theta})}{T} + \frac{1}{T}\sum_{t=0}^{T-1}\I(X_{t+1};\theta|\tilde{\theta}, H_t)\\
        & \leq \inf_{\tilde{\theta}}\ \frac{\I(\theta;\tilde{\theta})}{T} + \frac{1}{T}\sum_{t=0}^{T-1}\I(X_{t+1};\theta|\tilde{\theta}, H_t)\\
        & \leq \inf_{\epsilon \geq 0}\ \frac{\H_{\epsilon, T}(\theta)}{T} + \epsilon\\
    \end{align*}

    Suppose that $\I(H_T;\theta) < \H_{\epsilon, T}$.  Let $\tilde{\theta} = \tilde{H}_T \notin \tilde{\Theta}_{\epsilon, T}$ where $\tilde{H}_T$ is another history sampled in the same manner as $H_T$.
    \begin{align*}
        \I(H_T;\theta)
        & = \sum_{t=0}^{T-1}\I(X_{t+1};\theta|H_t)\\
        & \overset{(a)}{\geq} \sum_{t=0}^{T-1}\I(X_{t+1};\theta|\tilde{H}_{t}, H_t)\\
        & = \sum_{t=0}^{T-1}\I(X_{t+1};\theta|\tilde{\theta}, H_t)\\
        & \overset{(b)}{\geq} \epsilon T,
    \end{align*}
    where $(a)$ follows from the fact that conditioning reduces entropy and that $X_{t+1}\perp \tilde{H}_t|(\theta, H_t)$ and $(b)$ follows from the fact that $\tilde{\theta} \notin \tilde{\Theta}_{\epsilon, T}$.  Therefore, for all $\epsilon \geq 0$, $\I(H_{T};\theta) \geq \min\{ H_{\epsilon, T}, \epsilon T\}$.  The result follows.
\end{proof}

\subsection{Logistic Regression}\label{apdx:log_reg}
We introduce a simple logistic regression problem as a concrete instance to demonstrate an application of the general aforementioned results.  We assume that $X_0 = \bar{X}_0$ and $X_t = (Y_{t},\bar{X}_{t})$ for all $t \geq 1$.  The ``inputs'' $(\bar{X}_0, \ldots, \bar{X}_{T})$ are generated according to an iid random process for which $X_j\sim\normal(0, I_d)$. Meanwhile, we assume that $Y_{t+1}$ is generated by the following process:
$$Y_{t+1} = \begin{cases}
    1 & \text{ w.p. } \frac{1}{1+e^{-\theta^\top X_t}}\\
    -1 & \text{ otherwise }\\
\end{cases},$$
where $\theta$ denotes the parameters of the logistic model and we assume the prior distribution $\Pr(\theta\in\cdot) = \text{Unif}(\{\nu\in\Re^d: \|\nu\|_2 \leq 1 \})$.

In this environment, $\theta$ is the only unknown quantity and as such, the distributions of all random variables are \emph{known} to the algorithm designer.  In this example, the sequence is iid once conditioned on $\theta$.  We begin with this example for simplicity and to demonstrate that our analytical tools are general enough to subsume the analysis of supervised learning from iid data.

\begin{restatable}{theorem}{logReg}{\bf (logistic regression Bayesian error bounds)}
    For all $d, T\in\Z_{++}$, if $\theta, H_T$ follow the logistic regression environment, then 
    $$\Lc_T\ \leq\ \frac{d}{2T}\left(1+\ln\left(1+\frac{T}{4d}\right)\right).$$
\end{restatable}
\begin{proof}
    From Theorem \ref{th:bayes_rd}, it suffices to upper bound the rate-distortion function.
    Let $\tilde{\theta} = \theta + Z$ where $Z\perp \theta$ and $Z\sim\mathcal{N}(0, 8\epsilon/d)$. Then,
    \begin{align*}
        & \I(Y;\theta|\tilde{\theta}, X)\\
        & = \E\left[\KL(\Pr(Y\in\cdot|\theta, X) \| \Pr(Y\in\cdot|\tilde{\theta}, X))\right]\\
        & \overset{(a)}{\leq} \E\left[\KL(\Pr(Y\in\cdot|\theta, X) \| \Pr(Y\in\cdot|\theta\leftarrow\tilde{\theta}, X))\right]\\
        & = \E\left[\frac{\ln\left(\frac{e^{-\tilde{\theta}^\top X}}{e^{-\theta^\top X}}\right)}{1+e^{-\theta^\top X}} + \frac{\ln\left(\frac{e^{\theta^\top X}}{e^{\tilde{\theta}^\top X}}\right)}{1+e^{\theta^\top X}}\right]\\
        & \overset{(b)}{\leq} \frac{\E\left[\left(\theta^\top X - \tilde{\theta}^\top X\right)^2\right]}{8}\\
        & = \frac{\E\left[\|\theta-\tilde{\theta}\|^2_2\right]}{8}\\
        & = \epsilon,
    \end{align*}
    where $(a)$ follows from Lemma \ref{le:bayes_optimal_seq} and $(b)$ follows from the fact that for all $x,y\in\Re$,
    $$\frac{\ln\left(\frac{1+e^{-y}}{1+e^{-x}}\right)}{1+e^{-x}} + \frac{\ln\left(\frac{1+e^{y}}{1+e^{x}}\right)}{1+e^{x}} \leq (x-y)^2.$$
    Therefore, $\theta \in \Theta_\epsilon$ so it suffices to upper bound the rate $\I(\theta;\tilde{\theta})$.
    \begin{align*}
        \I(\theta;\tilde{\theta})
        & = \diffentropy(\tilde{\theta}) - \diffentropy(\tilde{\theta}|\theta)\\
        & = \diffentropy(\tilde{\theta}) - \diffentropy(Z|\theta)\\
        & = \diffentropy(\tilde{\theta}) - \diffentropy(Z)\\
        & \leq \frac{d}{2}\ln\left(2\pi e \left(\frac{1+8\epsilon}{d}\right)\right) - \frac{d}{2}\ln\left(2\pi e \frac{8\epsilon}{d}\right)\\
        & = \frac{d}{2}\ln\left(1 + \frac{1}{8\epsilon}\right).
    \end{align*}
    Therefore,
    \begin{align*}
        \Lc_n
        & \overset{(a)}{\leq}  \inf_{\epsilon\geq 0}\left(\frac{d}{2n}\ln\left(1+\frac{1}{8\epsilon}\right) + \epsilon\right)\\
        & \overset{(b)}{\leq} \frac{d}{2n}\ln\left(1+\frac{n}{4d}\right) + \frac{d}{2n},
    \end{align*}
    where $(a)$ follows from Theorem \ref{th:bayes_rd} and $(b)$ follows by setting $\epsilon = d/(2n)$.
\end{proof}

As one would expect, the above result establishes that the Bayesian error of an optimal learning algorithm is $\mathcal{O}(\frac{d}{n}\log\frac{n}{d})$.  The proof illustrates a common technique for bounding the rate-distortion function i.e. considering a compression $\tilde{\theta} = \theta + Z$ where $Z$ is independent zero-mean Gaussian noise with tunable variance. In the following section, we use the same set of tools to analyze a much more complex supervised learning problem involving a sequence generated by a deep transformer model.

\subsection{Transformers}

\begin{lemma}\label{le:real_input_inequality}
    For all $L\in\mathbb{Z}_{++}$ and $i\in \{1, \ldots, L\}$, if $\theta_i \perp \theta_j$, $\tilde{\theta}_i \perp \tilde{\theta}_j$, and $\theta_i \perp \tilde{\theta}_j$ for $i\neq j$, then
    $$\I(X_{t+1};\theta_{i}|\theta_{i+1:L}, \tilde{\theta}_{1:i}, H_t) \leq \I(H_{t+1};\theta_{i}|\theta_{i+1:L}, \theta_{1:i-1}, \tilde{\theta}_i, X_0).$$
\end{lemma}
\begin{proof}
    \begin{align*}
        \I(X_{t+1};\theta_{i}|\theta_{i+1:L}, \tilde{\theta}_{1:i}, H_{t})
        & \overset{(a)}{=} \I(H_{t+1}, \tilde{\theta}_{1:i-1};\theta_{i}|\theta_{i+1:L}, \tilde{\theta}_{i}) - \I(H_{t},\tilde{\theta}_{1:i-1} ;\theta_{i}|\theta_{i+1:L}, \tilde{\theta}_{i})\\
        & \overset{(b)}{\leq} \I(H_{t+1}, \tilde{\theta}_{1:i-1};\theta_{i}|\theta_{i+1:L}, \tilde{\theta}_{i}, X_0)\\
        & \overset{(c)}{\leq} \I(H_{t+1}, \theta_{1:i-1};\theta_{i}|\theta_{i+1:L}, \tilde{\theta}_{i}, X_0)\\
        & \overset{(d)}{=} \I(H_{t+1}, \theta_{1:i-1};\theta_{i}|\theta_{i+1:L}, \tilde{\theta}_{i}, X_0) - \I(\theta_{1:i-1};\theta_{i}|\theta_{i+1:L}, \tilde{\theta}_{i}, X_0)\\
        & \overset{(e)}{=} \I(H_{t+1};\theta_{i}|\theta_{i+1:L}, \theta_{1:i-1}, \tilde{\theta}_{i}, X_0)\\
    \end{align*}
    where $(a)$ follows from the chain rule of mutual information, $(b)$ follows from the independence assumptions, $(c)$ follows from the data processing inequality applied to the markov chain $\theta_i \perp \tilde{\theta}_{1:i-1}|(H_{t+1}, \theta_{i+1:L}, \theta_{1:i-1}, X_{0}^{K})$, $(d)$ follows from the fact that $\I(\theta_{1:i-1};\theta_i|\theta_{i+1:L},\tilde{\theta}_i, X_0) = 0$, and $(e)$ follows from the chain rule of mutual information.
\end{proof}

\begin{lemma}{\bf (transformer layer Lipschitz constant)}\label{le:transformer_lipschitz}
    For all $d,r,K\in\Z_{++}$, 
    $$\E\left[\|f_{\theta_{i}}(X) - f_{\theta_{i}}(\tilde{X})\|^2_F|X, \tilde{X}\right]\ \overset{a.s.}{\leq}\ 2(K+K^2)\cdot \|X-\tilde{X}\|^2_F.$$    
\end{lemma}
\begin{proof}
    Take all equality and inequality below to hold almost surely.
    \begin{align*}
        &\ \E\left[\|f_{i}(X) - f_{i}(\tilde{X})\|^2_F| X, \tilde{X}\right]\\
        & = \E\left[\left\|\text{Clip}\left(V_iX\sigma\left(\frac{X^\top A_i X}{\sqrt{r}}\right)\right) - \text{Clip}\left(V_i\tilde{X}\sigma\left(\frac{\tilde{X}^\top A_i \tilde{X}}{\sqrt{r}}\right)\right)  \right\|^2_F \bigg| X, \tilde{X}\right]\\
        & \overset{(a)}{\leq} \E\left[\left\|V_iX\sigma\left(\frac{X^\top A_i X}{\sqrt{r}}\right) - V_i\tilde{X}\sigma\left(\frac{\tilde{X}^\top A_i \tilde{X}}{\sqrt{r}}\right)  \right\|^2_F \bigg| X, \tilde{X}\right]\\
        & \overset{(b)}{=}
        \E\left[\sum_{k=1}^{K} \left\|V_{i} \left(X\sigma\left(\frac{X^\top A_i X_{k}}{\sqrt{r}}\right) - \tilde{X}\sigma\left(\frac{\tilde{X}^\top A_i \tilde{X}_k}{\sqrt{r}}\right)\right)\right\|^2_2 \bigg| X, \tilde{X}\right]\\
        & = 
        \E\left[\sum_{k=1}^{K} \left(X\sigma\left(\frac{X^\top A_i X_k}{\sqrt{r}}\right) - \tilde{X}\sigma\left(\frac{\tilde{X}^\top A_i \tilde{X}_k}{\sqrt{r}}\right)\right)^\top V_{i}^\top V_{i} \left(X\sigma\left(\frac{X^\top A_i X_k}{\sqrt{r}}\right) - \tilde{X}\sigma\left(\frac{\tilde{X}^\top A_i \tilde{X}_k}{\sqrt{r}}\right)\right) \bigg| X, \tilde{X} \right]\\
        & = \E\left[\sum_{k=1}^{K} \left(X\sigma\left(\frac{X^\top A_i X_k}{\sqrt{r}}\right) - \tilde{X}\sigma\left(\frac{\tilde{X}^\top A_i \tilde{X}_k}{\sqrt{r}}\right)\right)^\top\left(X\sigma\left(\frac{X^\top A_i X_k}{\sqrt{r}}\right) - \tilde{X}\sigma\left(\frac{\tilde{X}^\top A_i \tilde{X}_k}{\sqrt{r}}\right)\right) \bigg| X, \tilde{X} \right]\\
        & \leq \sum_{k=1}^{K}\E\left[2\left\|X\sigma\left(\frac{X^\top A_i X_k}{\sqrt{r}}\right) - \tilde{X}\sigma\left(\frac{X^\top A_i X_k}{\sqrt{r}}\right)\right\|^2_2 + 2\left\|\tilde{X}\sigma\left(\frac{X^\top A_i X_k}{\sqrt{r}}\right) - \tilde{X}\sigma\left(\frac{\tilde{X}^\top A_i \tilde{X}_k}{\sqrt{r}}\right)\right\|^2_2 \bigg| X, \tilde{X}\right]\\
        & \overset{(c)}{\leq} \sum_{k=1}^{K} \E\left[2\|X-\tilde{X}\|^2_F + \frac{2K}{r}\left\|X^\top A_i X_k - \tilde{X}^\top A_i \tilde{X}_k \right\|^2_2\big| X, \tilde{X}\right]\\
        & \overset{(d)}{\leq} \sum_{k=1}^{K} \E\left[2\|X-\tilde{X}\|^2_F + 2K^2\left\|X_k - \tilde{X}_k \right\|^2_2 \big| X, \tilde{X}\right]\\
        & = 2\left(K + K^2\right)\cdot \|X-\tilde{X}\|^2_F,
    \end{align*}
    where $(a)$ follows from the fact that Clip is a contraction mapping, where in $(b)$, $X_k$ denotes the $k$th column of $X\in\Re^{d\times K}$, $(c)$ follows from the fact that softmax is $1$-Lipschitz and $(d)$ follows from the fact that for all $k$, $\|\tilde{X}_{k}\|^2_2 \leq 1$.
\end{proof}

\begin{lemma}\label{le:tsfm_dist_part}
    For all $d, r, K \in \Z_{++}$ and $\epsilon \geq 0$, if $V\in\Re^{d\times d}$ consists of elements distributed iid $\normal(0, 1/d)$, $A\in\Re^{r\times r}$ consists of elements distributed $\normal(0, 1)$, 
    $\E[\|V-\tilde{V}\|^2_F] \leq \epsilon$, and $\E[\|A-\tilde{A}\|^2_F] \leq \epsilon/r$, then
    $$\E\left[\|f_{\theta}(X) - f_{\tilde{\theta}}(X)\|^2_F\right]\ \leq\ 2K^2\epsilon\left(1 + Kd\right),$$
    where $\theta = (V,A), \tilde{\theta} = (\tilde{V}, \tilde{A})$.
\end{lemma}
\begin{proof}
    \begin{align*}
        &\ \E\left[\left\|f_{\theta}(X) - f_{\tilde{\theta}}(X)\right\|^2_F\right]\\
        & \leq \E\left[\sup_{x\in\mathcal{X}}\ \|f_{\theta}(x) - f_{\tilde{\theta}}(x)\|^2_F\right]\\
        & = \E\left[\sup_{x\in\mathcal{X}}\left\|V x \sigma\left(\frac{x^\top A x}{\sqrt{r}}\right) - \tilde{V} x \sigma\left(\frac{x^\top \tilde{A} x}{\sqrt{r}}\right)\right\|^2_F\right]\\
        & \overset{(a)}{\leq} 2\E\left[\sup_{x\in\mathcal{X}}\left\|\left(V - \tilde{V}\right) x \sigma\left(\frac{x^\top \tilde{A} x}{\sqrt{r}}\right)\right\|^2_F\right] + 2\E\left[\sup_{x\in\mathcal{X}} \left\|Vx\left(\sigma\left(\frac{x^\top A x}{\sqrt{r}}\right) 
 - \sigma\left(\frac{x^\top \tilde{A} x}{\sqrt{r}}\right)\right)\right\|^2_F\right]\\
        & \overset{(b)}{\leq} 2\E\left[\sup_{x\in\mathcal{X}}\left\|V- \tilde{V}\right\|^2_F \left\| x \sigma\left(\frac{x^\top \tilde{A} x}{\sqrt{r}}\right) \right\|^2_F\right] + 2\E\left[\sup_{x\in\mathcal{X}} \left\|V\right\|^2_F\left\|x\sigma\left(\frac{x^\top A x}{\sqrt{r}}\right) 
 - x\sigma\left(\frac{x^\top \tilde{A} x}{\sqrt{r}}\right)\right\|^2_F\right]\\
        & \overset{(c)}{\leq} 2\epsilon \cdot \sup_{x\in\Xc} \|x\|^2_F \cdot \left\|\sigma\left(\frac{x^\top \tilde{A} x}{\sqrt{r}}\right)\right\|^2_F + 2\E\left[\|V\|^2_F\cdot \sup_{x\in\Xc} \|x\|^2_F\left\|\sigma\left(\frac{x^\top A x}{\sqrt{r}}\right) 
 - \sigma\left(\frac{x^\top \tilde{A} x}{\sqrt{r}}\right)\right\|^2_F\right]\\
        & \overset{(d)}{\leq} 2\epsilon K^2 + 2\E\left[\frac{dK}{r}\cdot \sup_{x\in\Xc} \left\|x^\top A x 
 - x^\top \tilde{A} x\right\|^2_F\right]\\
        & \overset{(e)}{\leq} 2\epsilon K^2 + \frac{2Kd}{r}\cdot \E\left[\sum_{i=1}^{K}\sum_{j=1}^{K}\left(x_i^\top(A-\tilde{A}) x_j\right)^2 \right]\\
        & \leq 2\epsilon K^2 + \frac{2Kd}{r}\cdot \E\left[\sup_{x\in\Xc}\ \sum_{i=1}^{K}\sum_{j=1}^{K}\left\|A-\tilde{A} \right\|^2_F \right]\\
        & = 2\epsilon K^2 + 2K^3d\epsilon,
    \end{align*}
    where $(a)$ follows from the fact that $\|a + b\|^2_F \leq 2\|a\|^2_F + 2\|b\|^2_F$ for all matrices $a, b$, $(b)$ follows from the fact that $\|ab\|^2_F \leq \|a\|^2_\sigma\|b\|^2_F$ and $\|a\|^2_\sigma \leq \|a\|^2_F$ for all matrices $a, b$, $(c)$ follows from the fact that $\E\left[\|V - \tilde{V}\|^2_F\right] = \epsilon$, $(d)$ follows from the fact that $\E[\| V\|^2_F] = d$, and the fact that softmax is $1$-Lipschitz, and where in $(e)$, $x_i$ denotes the $i$th column of matrix $x$.
\end{proof}

\begin{lemma}{\bf (sequence transformer distortion bound)}\label{le:seq_tsfm_dist}
    For all $d,r,t,K,L \in \Z_{++}$, $0 \leq \epsilon\leq 2d$, and $i \leq L$, if $\tilde{\theta}_i = (\tilde{V}_i, \tilde{A}_i)$ for which $\tilde{V}_i = V_i + Z^{V}_i, \tilde{A}_i = A_i + Z^{A}_i$, $(V_i, A_i)\perp (Z^{V}_i, Z^{A}_i)$, $Z^V_i$ consists of elements distributed iid $\normal(0, \epsilon/d^2)$, and $Z^A_i$ consists of elements distributed iid $\normal(0, \epsilon/r)$, then
    $$\I(X_{t+1};\theta_i|\theta_{i+1:L}, \theta_{1:i-1}, \tilde{\theta}_i, H_t) \leq \epsilon K d \left(2K+2K^2\right)^{L-i+1}.$$
\end{lemma}
\begin{proof}
    \begin{align*}
        \I(X_{t+1};\theta_i|\theta_{i+1:L}, \theta_{1:i-1}, \tilde{\theta}_i, H_t)
        & = \E\left[\KL\left(\Pr(X_{t+1}\in\cdot|\theta_{1:L}, H_t)\ \|\ \Pr(X_{t+1}\in\cdot|\theta_{i+1:L}, \theta_{1:i-1}, \tilde{\theta}_i, H_t)\right)\right]\\
        & \overset{(a)}{\leq} \E\left[\KL\left(\Pr(X_{t+1}\in\cdot|\theta_{1:L}, H_t)\ \|\ \Pr(X_{t+1}\in\cdot|\theta_{i+1:L}, \theta_{1:i-1}, \theta_i \leftarrow \tilde{\theta}_i, H_t)\right)\right]\\
        & \overset{(b)}{\leq} \E\left[\left\|f_{\theta_{1:L}}(H_t) - f_{\theta_{i+1:L}}(f_{\tilde{\theta}_i}(f_{\theta_{1:i-1}}(H_t)))\right\|^2_2\right]\\
        & = \E\left[\left\|f_{\theta_{i+1:L}}(f_{\theta_{i}}(f_{\theta_{1:i-1}}(H_t))) - f_{\theta_{i+1:L}}(f_{\tilde{\theta}_i}(f_{\theta_{1:i-1}}(H_t)))\right\|^2_2\right]\\
        & \overset{(c)}{=} \E\left[\left\|f_{\theta_{i+1:L}}(f_{\theta_{i}}(U_{t,i-1})) - f_{\theta_{i+1:L}}(f_{\tilde{\theta}_i}(U_{t,i-1}))\right\|^2_2\right]\\
        & = \E\left[\frac{\left\|f_{\theta_{i+1:L}}(f_{\theta_{i}}(U_{t,i-1})) - f_{\theta_{i+1:L}}(f_{\tilde{\theta}_i}(U_{t,i-1}))\right\|^2_2}{\left\|f_{\theta_{i}}(U_{t,i-1}) - f_{\tilde{\theta}_i}(U_{t,i-1})\right\|^2_2} \cdot \left\|f_{\theta_{i}}(U_{t,i-1}) - f_{\tilde{\theta}_i}(U_{t,i-1})\right\|^2_2\right]\\
        & \overset{(d)}{\leq} \E\left[\left(2K+2K^2\right)^{L-i}\cdot\left\|f_{\theta_i}(U_{t,i-1}) - f_{\tilde{\theta}_i}(U_{t,i-1})\right\|^2_2\right]\\
        & \overset{(e)}{\leq} \left(2K+2K^2\right)^{L-i} \cdot \epsilon K\left(2K + 2K^2d\right)\\
        & \leq \epsilon K d \left(2K+2K^2\right)^{L-i+1},
    \end{align*}
    where $(a)$ follows from Lemma \ref{le:bayes_optimal_seq}, $(b)$ follows from Lemma \ref{le:sq_kl}, where in $(c)$, $U_{t,i} = f_{\theta_{1:i}}(H_t)$, $(d)$ follows from Lemma \ref{le:transformer_lipschitz}, and $(e)$ follows from Lemma \ref{le:tsfm_dist_part}.
\end{proof}

\begin{lemma}{\bf (sequence transformer distortion bound)}\label{le:seq_tsfm_dist_ub}
    For all $d,r,K,L \in \Z_{++}$, $0\leq \epsilon \leq 2d$ and $i \leq L$, if $\tilde{\theta}_i = (\tilde{V}_i, \tilde{A}_i)$ for which $\tilde{V}_i = V_i + Z^{V}_i, \tilde{A}_i = A_i + Z^{A}_i$, $(V_i, A_i)\perp (Z^{V}_i, Z^{A}_i)$, $Z^V_i$ consists of elements distributed iid $\normal(0, \epsilon/d^2)$, and $Z^A_i$ consists of elements distributed iid $\normal(0, \epsilon/r)$, then
    $$\I(H_{t+1};\theta_i|\theta_{i+1:L}, \theta_{1:i-1}, \tilde{\theta}_i)\ \leq\ \epsilon K (t+1)d \left(2K+2K^2\right)^{L-i+1}.$$
\end{lemma}
\begin{proof}
    \begin{align*}
        \I(H_{t+1};\theta_i|\theta_{i+1:L}, \theta_{1:i-1}, \tilde{\theta}_i)
        & = \sum_{k=0}^{t} \I(X_{k+1};\theta_{i}|\theta_{i+1:L}, \theta_{1:i-1}, \tilde{\theta}_i, H_k)\\
        & \overset{(a)}{\leq} \sum_{k=0}^{t} \epsilon K d \left(2K+2K^2\right)^{L-i+1}\\
        & = \epsilon K(t+1)d\left(2K+2K^2\right)^{L-i+1},
    \end{align*}
    where $(a)$ follows from Lemma \ref{le:seq_tsfm_dist}.
\end{proof}

\begin{lemma}
    For all $d,r,t,K,L\in\Z_{++}$, if for all $i \leq L$, $\tilde{\theta}_i = (\tilde{V}_i, \tilde{A}_i)$ for which $\tilde{V}_i = V_i + Z^{V}_i, \tilde{A}_i = A_i + Z^{A}_i$, $(V_i, A_i)\perp (Z^{V}_i, Z^{A}_i)$, $Z^V_i$ consists of elements distributed iid $\normal(0, \epsilon/d^2)$, $\tilde{A}_i = A_i + Z^{A}_i$, $(V_i, A_i)\perp (Z^{V}_i, Z^{A}_i)$, $Z^V_i$ consists of elements distributed iid $\normal(0, \epsilon/d^2)$, and $Z^A_i$ consists of elements distributed iid $\normal(0, \epsilon/r)$, then
    $$\I(X_{t+1};\theta_{1:L}|\tilde{\theta}_{1:L}, H_t)\ \leq \ \epsilon KL(t+1)d\left(2K+2K^2\right)^{L}.$$
\end{lemma}
\begin{proof}
    \begin{align*}
        \I(X_{t+1};\theta_{1:L}|\tilde{\theta}_{1:L}, H_t)
        & = \sum_{i=1}^{L}\ \I(X_{t+1};\theta_{i}|\tilde{\theta}_{1:L},\theta_{i+1:L} H_t)\\
        & \overset{(a)}{\leq} \sum_{i=1}^{L}\ \I(H_{t+1};\theta_{i}|\theta_{i+1:L},\theta_{1:i-1}, \tilde{\theta}_{i}, X_0)\\
        & \overset{(b)}{\leq} \epsilon KL(t+1)d\left(2K+2K^2\right)^{L},
    \end{align*}
    where $(a)$ follows from Lemma \ref{le:real_input_inequality}, and $(b)$ follows from Lemma \ref{le:seq_tsfm_dist_ub}.
\end{proof}

\tsfmRd*
\begin{proof}
    Let $\epsilon = \frac{\epsilon'}{dKLT(2K+2K^2)^L}$.
    \begin{align*}
        \I(\theta_{1:L};\tilde{\theta}_{1:L})
        & = \diffentropy(\tilde{\theta}_{1:L}) - \diffentropy(\tilde{\theta}_{1:L}|\theta_{1:L})\\
        & = \sum_{i=1}^{L} \diffentropy(\tilde{\theta}_i) - \diffentropy(\tilde{\theta}_i|\theta_i)\\
        & = L\left(\diffentropy(\tilde{V}_i) - \diffentropy(\tilde{V}_i|V_i) + \diffentropy(\tilde{A}_i) - \diffentropy(\tilde{A}_i|A_i) \right)\\
        & \leq L\left( \frac{d^2}{2}\log\left(1 + \frac{d^2KLT(2K+2K^2)^L)}{\epsilon'}\right) + \frac{r^2}{2}\log\left(1 + \frac{drKLT(2K+2K^2)^L}{\epsilon'}\right)\right)\\
        & \overset{(a)}{\leq} \frac{(d^2+r^2)L^2\log\left(2K+2K^2\right)}{2} + \frac{(d^2+r^2)L\log\left(\frac{2dKLT}{\epsilon'}\right)}{2} + \frac{d^2L\log(d) + r^2L \log(r)}{2},\\
        & \leq \frac{(d^2+r^2)L^2\log\left(2K+2K^2\right)}{2} + \frac{(d^2+r^2)L\log\left(\frac{2\max\{d,r\} \cdot dKLT}{\epsilon'}\right)}{2}\\
        & \leq \frac{(d^2+r^2)L^2\log\left(2K+2K^2\right)}{2} + \frac{(d^2+r^2)L\log\left(\frac{2\max\{d,r\} \cdot dKLT}{\epsilon'}\right)}{2},
    \end{align*}
    where $(a)$ holds for $\epsilon' < d^2KLT(2K+2K^2)^L$.
    Setting $\epsilon' = (d^2+r^2)L^2\log(2K+2K^2)/2T$ gives the result. 
\end{proof}

\section{Meta-Learning from Sequential Data}
\metaBayesOpt*
\begin{proof}
    In the below proof take all equality to hold \emph{almost surely}.
    \begin{align*}
        &\ \E\left[-\ln P_{m,t}\left(X^{(m)}_{t+1}\right)|H_{m,t}\right]\\
        & = \E\left[-\ln \hat{P}_{m,t}\left(X^{(m)}_{t+1}\right) + \ln\frac{\hat{P}_{m,t}(X^{(m)}_{t+1})}{P_{m,t}(X^{(m)}_{t+1})}\Big| H_{m,t}\right]\\
        & = \E\left[-\ln \hat{P}_{m,t}\left(X^{(m)}_{t+1}\right)\Big| H_{m,t}\right] +\KL(\hat{P}_{m,t}\|P_{m,t}).
    \end{align*}
    The result follows from the fact that $\KL(\hat{P}_{m,t}\|P_{m,t}) > 0$ for all $P_{m,t} \neq \hat{P}_{m,t}$.
\end{proof}

\metaSeqBayesError*
\begin{proof}
    \begin{align*}
        \L_{M,T}
        & = \frac{1}{MT}\sum_{m=1}^{M}\sum_{t=0}^{T-1}\ \E\left[- \ln\hat{P}_{m,t}(X^{(m)}_{t+1})\right]\\
        & = \frac{1}{MT}\sum_{m=1}^{M}\sum_{t=0}^{T-1}\ \H(X_{t+1}^{(m)}|\theta_m, H_t^{(m)}) + \E\left[\KL\left(\Pr(X_{t+1}^{(m)}\in\cdot|\psi,\theta_m, H^{(m)}_{t})\|\Pr(X_{t+1}^{(m)}\in\cdot|H_{m,t})\right)\right]\\
        & = \frac{1}{MT}\sum_{m=1}^{M}\sum_{t=0}^{T-1}\ \I(X_{t+1}^{(m)}; \psi, \theta_m|H_{m,t}) + \H(X_{t+1}^{(m)}|\theta_m, H_t^{(m)})\\
        & = \frac{1}{MT}\sum_{m=1}^{M}\I(H^{(m)}_{T};\psi, \theta_m|H_{m-1,T}) + \H(H_{T}^{(m)}|\theta_m)\\
        & = \frac{1}{MT}\sum_{m=1}^{M}\I(H^{(m)}_{T};\psi|H_{m-1,T}) + \I(H_{T}^{(m)};\theta_m|\psi, H_{m-1,T}) + \H(H_{T}^{(m)}|\theta_m)\\
        & = \frac{\I(H_{M,T};\psi)}{MT} + \frac{\I(H^{(1)}_{T};\theta_1|\psi)}{T} + \frac{1}{MT}\sum_{m=1}^{M} \H(H_{T}^{(m)}|\theta_m).
    \end{align*}
\end{proof}

\metaRD*
\begin{proof}
    We begin by showing the upper bound:
    \begin{align*}
        \Lc_{M,T}
        & = \frac{\I(H_{M,T};\psi)}{MT} + \frac{\I(D_m;\theta_m|\psi)}{T}\\
        & = \frac{\I(H_{M,T};\psi,\tilde{\psi})}{MT} + \frac{\I(D_m;\theta_m,\tilde{\theta}_m|\psi)}{T}\\
        & = \frac{\I(H_{M,T};\tilde{\psi})}{MT} + \frac{\I(H_{M,T};\psi|\tilde{\psi})}{MT} + \frac{\I(D_m;\theta_m,\tilde{\theta}_m|\psi)}{T}\\
        & \overset{(a)}{\leq} \frac{\I(\psi;\tilde{\psi})}{MT} + \frac{\I(H_{M,T};\psi|\tilde{\psi})}{MT} + \frac{\I(D_m;\theta_m,\tilde{\theta}_m|\psi)}{T}\\
        & = \frac{\I(\psi;\tilde{\psi})}{MT} + \frac{\I(H_{M,T};\psi|\tilde{\psi})}{MT} + \frac{\I(D_m;\tilde{\theta}_m|\psi)}{T} + \frac{\I(D_m;\theta_m|\tilde{\theta}_m, \psi)}{T}\\
        & \overset{(b)}{\leq} \frac{\I(\psi;\tilde{\psi})}{MT} + \frac{\I(H_{M,T};\psi|\tilde{\psi})}{MT} + \frac{\I(\theta_m;\tilde{\theta}_m|\psi)}{T} + \frac{\I(D_m;\theta_m|\tilde{\theta}_m, \psi)}{T}\\
        & \overset{(c)}{\leq} \frac{\H_{\epsilon,M,T}(\psi)}{MT} + \epsilon + \frac{\H_{\epsilon',M,T}(\tilde{\theta}_m|\psi)}{T} + \epsilon',
    \end{align*}
    where $(a)$ and $(b)$ follow from the data processing inequality and $(c)$ follows from the definition of the rate-distortion functions.  The upper bound follows from the fact that inequality $(c)$ holds for all $\epsilon \geq 0$.
    
    We now prove the lower bound.
    Suppose that $\I(H_{M,T};\psi) < \H_{\epsilon,M,T}(\psi)$
    Let $\tilde{\psi} = \tilde{H}_{M,T} \notin \tilde{\Psi}_{\epsilon,M,T}$ where $\tilde{H}_{M,T}$ is another history sampled in the same manner as $H_{M,T}$.
    \begin{align*}
        \I(H_{M,T};\psi)
        & = \sum_{m=1}^{M}\sum_{t=0}^{T-1}\I(X^{(m)}_{t+1};\psi|H_{m,t})\\
        & \overset{(a)}{\geq} \sum_{m=1}^{M}\sum_{t=0}^{T-1}\I(X^{(m)}_{t+1};\psi|\tilde{H}_{M,T}, H_{m,t})\\
        & = \sum_{m=1}^{M}\sum_{t=0}^{T-1}\I(X^{(m)}_{t+1};\psi|\tilde{\psi}, X^{(m)}_1,\ldots, X^{(m)}_t)\\
        & \overset{(b)}{\geq} \epsilon MT,
    \end{align*}
    where $(a)$ follows from the fact that conditioning reduces entropy and that $X^{(m)}_{t+1}\perp \tilde{H}_{M,T}|(\psi, H_{m,t})$ and $(b)$ follows from the fact that $\tilde{\psi} \notin \tilde{\Psi}_{\epsilon,M, T}$.  Therefore, for all $\epsilon \geq 0$, $\I(H_{M,T};\psi) \geq \min\{ H_{\epsilon, M, T}(\psi), \epsilon MT\}$.

    Suppose that $\I(H_{T}^{(m)};\theta_{m}|\psi) < \H_{\epsilon,T}(\theta_{m}|\psi)$.  Let $\tilde{\theta}_{m} = \tilde{D}_m \notin \tilde{\Theta}_{\epsilon,T}$ where $\tilde{D}_{m}$ is another history sampled in the same manner as $D_m$.
    \begin{align*}
        \I(D_m;\theta_m|\psi)
        & = \sum_{t=0}^{T-1}\I(X^{(m)}_{t+1};\theta_m|X^{(m)}_1,\ldots, X^{(m)}_t, \psi)\\
        & \overset{(a)}{\geq} \sum_{t=0}^{T-1}\I(X^{(m)}_{t+1};\theta_m|\tilde{D}_{m}, X^{(m)}_1,\ldots, X^{(m)}_t, \psi)\\
        & = \sum_{t=0}^{T-1}\I(X^{(m)}_{t+1};\theta_m|\tilde{\theta}_m, H_{m,t},\psi)\\
        & \overset{(b)}{\geq} \epsilon T,
    \end{align*}
    where $(a)$ follows from the fact that conditioning reduces entropy and that $X^{(m)}_{t+1}\perp \tilde{D}_{m}|(\psi, X^{(m)}_1,\ldots,X^{(m)}_t )$ and $(b)$ follows from the fact that $\tilde{\theta}_m \notin \tilde{\Theta}_{\epsilon,T}$.  Therefore, for all $\epsilon \geq 0$, $\I(D_m;\theta_{m}|\psi) \geq \min\{ H_{\epsilon, M, T}(\theta_m), \epsilon T\}$.  The lower bound follows as a result.
\end{proof}

\subsection{Linear Representation Learning Example}\label{apdx:lin_rep}
We introduce a simple linear representation learning problem as a concrete example of meta-learning to demonstrate our method of analysis.  Just as in the logistic regression example, the documents in this example consist of iid data but we begin with such an example for simplicity and to demonstrate this as a special case of meta-learning from sequences under our framework.

For all $d,r\in \mathbb{Z}_{++}$, we let $\psi:\Omega\mapsto\Re^{d\times r}$ be distributed uniformly over the set of $d\times r$ matrices with orthonormal columns. We assume that $d \gg r$. For all $i$, let $\xi_i:\Omega\mapsto\Re^{r}$ be distributed iid $\normal(0, I_r/r)$. We let $\theta_i = \psi\xi_i$ and hence $\psi$ induces a distribution on $\theta_i$. As for the observable data, for each $(i,j)$, let $X_j^{(i)} = \emptyset$ and $Y_{j+1}^{(i)}$ be drawn as according to the following probability law:
$$Y^{(i)}_{j+1} =
\begin{cases}
    1 & \text{w.p. } \sigma(\theta_{i})_1\\
    2 & \text{w.p. } \sigma(\theta_{i})_2\\
    \hdots & \\
    d & \text{w.p. } \sigma(\theta_{i})_d\\
\end{cases},$$
where $\sigma(\theta_{i})_j = e^{\theta_{i,j}}/\sum_{k=1}^{d} e^{\theta_{i,k}}$ denotes softmax.  Note that in this problem, the input $X$ does not influence the output $Y$. For each task $i$, the algorithm is tasked with estimating a vector $\theta_i$ from noisy observations $(Y_1^{(i)}, \ldots, Y_n^{(i)})$. By reasoning about data from previous tasks, the algorithm can estimate $\psi$ which reduces the burden of estimating $\theta_i$ to just estimating $\xi_i$ for each task.  This is significant given the assumption that $d \gg r$.  We now present the theoretical result.
\begin{restatable}{theorem}{linRepLearningBound}{\bf (linear representation learning Bayesian error bound)}
    For all $d,r,M,T\in\Z_{++}$,
    \begin{align*}
        \Lc_{M,T}
        & \leq \frac{dr\left(1+\log\left(1 + \frac{M}{r}\right)\right)}{2MT} + \frac{r\left(1+\log(1+\frac{2n}{r})\right)}{2T}.
    \end{align*}
\end{restatable}

The first term indicates the standard irreducible error. The second term indicates the statistical error incurred in the process of estimating $\psi$. Since $\psi\in \Re^{d\times r}$ and there are $m\times n$ data points in total which contain information about $\psi$.  The final term represents statistical error incurred in the process of estimating $\xi_1, \ldots, \xi_m$. Since each $\xi_i\in\Re^r$ and there are $n$ data points which contain information about each $\xi_i$ the $\tilde{O}(r/n)$ follows standard statistical intuition.

We note that this tightens a result shown in \citep{pmlr-v139-tripuraneni21a} which studies an almost identical problem.  Their proposed upper bound is $\tilde{\mathcal{O}}(\frac{dr^2}{MT} + \frac{r}{T})$ which contains an extra factor of $r$ in the meta-estimation error.

In the following, we will provide a result which requires a change of measure.  For all random variables $X:\Omega\mapsto \Xc, Y:\Omega\mapsto \Yc$ and realizations $y\in\Yc$, one may consider the distribution $\Pr(X\in\cdot|Y=y)$.  Let function $f(y) = \Pr(X\in\cdot|Y=y)$.  Then, for any random variable $Z:\Omega\mapsto\mathcal{Z}$ for which $\mathcal{Z} \subseteq \Yc$, we use $\Pr(X\in\cdot|Y \leftarrow Z)$ to denote $f(Z)$. 

\begin{lemma}{\bf(sq error upper bounds softmax KL-divergence)}\label{le:sq_kl}
    For all $d \in\Z_{++}$ and random vectors $\theta,\tilde{\theta}\in\Re^d$,
    $$\E\left[\sum_{l=1}^{d}\frac{e^{\theta_l}}{\sum_{k=1}^{d}e^{\theta_k}} \ln\frac{\frac{e^{\theta_l}}{\sum_{k=1}^{d}e^{\theta_k}}}{\frac{e^{\tilde{\theta}_l}}{\sum_{k=1}^{d}e^{\tilde{\theta}_k}}}\right] \ \leq\ \E\left[\|\tilde{\theta}-\theta\|^2_2\right].$$
\end{lemma}
\begin{proof}
    \begin{align*}     \E\left[\KL(\Pr(Y\in\cdot|\theta)\|\KL(Y\in\cdot|\theta\leftarrow\tilde{\theta}))\right]
        & = \E\left[\sum_{l=1}^{d}\frac{e^{\theta_l}}{\sum_{k=1}^{d}e^{\theta_k}} \ln\frac{\frac{e^{\theta_l}}{\sum_{k=1}^{d}e^{\theta_k}}}{\frac{e^{\tilde{\theta}_l}}{\sum_{k=1}^{d}e^{\tilde{\theta}_k}}}\right]\\
        & = \E\left[\sum_{l=1}^{d}\frac{e^{\theta_l}}{\sum_{k=1}^{d}e^{\theta_k}} \left(\ln\frac{e^{\theta_l}}{e^{\tilde{\theta}_l}} + \ln\frac{\sum_{k=1}^{d}e^{\tilde{\theta}_k}}{\sum_{k=1}^{d}e^{\theta_k}}\right)\right]\\
        & = \E\left[\ln\frac{\sum_{k=1}^{d}e^{\tilde{\theta}_k}}{\sum_{k=1}^{d}e^{\theta_k}}\right] + \E\left[\sum_{l=1}^{d}\frac{e^{\theta_l}}{\sum_{k=1}^{d}e^{\theta_k}} \ln\frac{e^{\theta_l}}{e^{\tilde{\theta}_l}}\right]\\
        & = \E\left[\ln\frac{\sum_{k=1}^{d}e^{\tilde{\theta}_k}}{\sum_{k=1}^{d}e^{\theta_k}}\right] + \E\left[\sum_{l=1}^{d}\frac{e^{\theta_l}}{\sum_{k=1}^{d}e^{\theta_k}} \left(\theta_l-\tilde{\theta}_l\right)\right]\\
        & = \E\left[\ln\frac{\sum_{k=1}^{d}e^{\tilde{\theta}_k}}{\sum_{k=1}^{d}e^{\theta_k}}\right] + \E\left[\sum_{l=1}^{d}\frac{e^{\theta_l}}{\sum_{k=1}^{d}e^{\theta_k}} \left(\theta_l-\tilde{\theta}_l\right)\right]\\
        & = \E\left[\ln\frac{\sum_{k=1}^{d}e^{\tilde{\theta}_k}}{\sum_{k=1}^{d}e^{\theta_k}}\right] + \E\left[\sum_{l=1}^{d}\frac{e^{\theta_l}}{\sum_{k=1}^{d}e^{\theta_k}} \left(\theta_l-\tilde{\theta}_l\right)\right]\\
        & \overset{(a)}{\leq} \E\left[\sum_{l=1}^{d}\frac{e^{\tilde{\theta}_l}}{\sum_{k=1}^{d}e^{\tilde{\theta}_k}} \ln \frac{e^{\tilde{\theta}_l}}{e^{\theta_l}}\right]+\E\left[\sum_{l=1}^{d}\frac{e^{\theta_l}}{\sum_{k=1}^{d}e^{\theta_k}} \left(\theta_l-\tilde{\theta}_l\right)\right]\\
        & = \E\left[\sum_{l=1}^{d}\frac{e^{\tilde{\theta}_l}}{\sum_{k=1}^{d}e^{\tilde{\theta}_k}} \left(\tilde{\theta}_l - \theta_l\right)\right] + \E\left[\sum_{l=1}^{d}\frac{e^{\theta_l}}{\sum_{k=1}^{d}e^{\theta_k}} \left(\theta_l-\tilde{\theta}_l\right)\right]\\
        & = \E\left[\sum_{l=1}^{d}\left(\frac{e^{\tilde{\theta}_l}}{\sum_{k=1}^{d}e^{\tilde{\theta}_k}}-\frac{e^{\theta_l}}{\sum_{k=1}^{d}e^{\theta_k}}\right)\left(\tilde{\theta}_l-\theta_l\right)\right]\\
        & \overset{(b)}{\leq} \E\left[\sum_{l=1}^{d}\left(\tilde{\theta}_l-\theta_l\right)^2\right]\\
        & = \E\left[\|\tilde{\theta}-\theta\|^2_2\right],
    \end{align*}
    where $(a)$ follows from the log-sum inequality and $(b)$ follows from the fact that the softmax function is $1$-Lipschitz.
\end{proof}

\begin{lemma}{\bf (rate upper bound)}\label{le:rate_ub}
    For all $d,r,m,n\in\Z_{++}$,
    $$\frac{\I(H_{m,n};\psi)}{mn}\ \leq\ \inf_{\epsilon \geq 0}\ \frac{dr\log\left(1 + \frac{1}{r\epsilon}\right)}{2mn} + \frac{r\log(1+d\epsilon)}{2n}.$$
\end{lemma}
\begin{proof}
    Let $\tilde{\psi} = \psi + Z$ where $Z \in \Re^{d\times k}$ is $Z\perp \psi$ and consists of elements which are distributed iid $\normal(0, \epsilon)$.
    \begin{align*}
        \frac{\I(H_{m,n};\psi)}{mn}
        & \overset{(a)}{=} \frac{\I(H_{m,n};\psi,\tilde{\psi})}{mn}\\
        & \overset{(b)}{=} \frac{\I(H_{m,n};\tilde{\psi})}{mn} + \frac{\I(H_{m,n};\psi|\tilde{\psi})}{mn}\\
        & = \frac{\I(H_{m,n};\tilde{\psi})}{mn} + \frac{\H(H_{m,n}|\tilde{\psi}) - \H(H_{m,n}|\psi)}{mn}\\
        & \overset{(c)}{=} \frac{\I(H_{m,n};\tilde{\psi})}{mn} + \frac{\sum_{i=1}^{m}\H(H_{n}^{(i)}|\tilde{\psi}, H_{i-1,n}) - m\cdot\H(H_n^{(1)}|\psi)}{mn}\\
        & \overset{(d)}{\leq} \frac{\I(H_{m,n};\tilde{\psi})}{mn} + \frac{m\cdot\diffentropy(H_{n}^{(1)}|\tilde{\psi}) - m\cdot\diffentropy(H_n^{(1)}|\psi)}{mn}\\
        & \overset{(e)}{\leq} \frac{\I(\psi;\tilde{\psi})}{mn} + \frac{\I(H^{(1)}_{n};\psi|\tilde{\psi})}{n}\\
        & \overset{(f)}{\leq} \frac{\I(\psi;\tilde{\psi})}{mn} + \frac{\I(\theta_1;\psi|\tilde{\psi})}{n},
    \end{align*}
    where $(a)$ follows from the fact that $H_{m,n}\perp \tilde{\psi}|\psi$, $(b)$ follows from the chain rule of mutual information, $(c)$ follows from the chain rule of mutual information and the fact that $H_{n}^{(i)}$ are iid $| \psi$, $(d)$ follows from the fact that conditioning reduces differential entropy, and $(e)/(f)$ both follow from the data processing inequality applied to the markov chains $\tilde{\psi}\perp H_{m,n}|\psi$ and $\psi \perp H_n^{(1)}|\theta_1, \tilde{\psi}$.

    We now bound the two above terms.

    \begin{align*}
        \frac{\I(\psi;\tilde{\psi})}{mn}
        & = \frac{\diffentropy(\tilde{\psi}) - \diffentropy(\tilde{\psi}|\psi)}{mn}\\
        & \leq \frac{\frac{dr}{2}\log\left(2\pi e \left(\epsilon + \frac{1}{r}\right)\right) - \frac{dr}{2}\log\left(2\pi e \epsilon\right)}{mn}\\
        & = \frac{dr\log\left(1 + \frac{1}{r\epsilon}\right)}{2mn},
    \end{align*}
    where $(a)$ follows from the maximum differential entropy of a random variable of fixed variance being upper bounded by a Gaussian random variable.

    Let $\theta_\delta = \theta_1 + \delta Z$ where $Z \sim \normal(0, I_d)$ and $Z \perp \theta_1$.
    \begin{align*}
        \frac{\I(\theta_1;\psi|\tilde{\psi})}{n}
        & = \frac{\I(\theta_1;\psi|\tilde{\psi})}{n}\\
        & = \frac{\E\left[\KL(\Pr( \theta_1\in\cdot|\psi)\|\Pr(\theta_1\in\cdot|\tilde{\psi}))\right]}{n}\\
        & \overset{(a)}{\leq} \frac{\E\left[\KL(\Pr(\theta_1\in\cdot|\psi)\|\Pr(\theta_\in\cdot|\psi\leftarrow\tilde{\psi}))\right]}{n}\\
        & \leq \frac{\E\left[\KL(\lim_{\delta\to 0}\Pr(\theta_\delta\in\cdot|\psi)\|\lim_{\delta\to 0}\Pr(\theta_\delta|\psi\leftarrow\tilde{\psi}))\right]}{n}\\
        & \overset{(b)}{=} \frac{1}{n}\E\left[\lim_{\delta\to 0 }\frac{1}{2}\log\left(\frac{\left|\delta I_d + \frac{\tilde{\psi}\tilde{\psi}^\top}{k}\right|}{\left|\delta I_d + \frac{\psi\psi^\top}{k}\right|}\right) - d + {\rm Tr}\left(\left(\delta I_d + \frac{\tilde{\psi}\tilde{\psi}^\top}{k}\right)^{-1}\left(\delta I_d + \frac{\psi\psi^\top}{k}\right)\right)\right]\\
        & \overset{(c)}{\leq} \frac{1}{n}\E\left[\lim_{\delta\to 0 }\frac{1}{2}\log\left(\frac{\left|\delta I_d + \frac{\tilde{\pi}\tilde{\psi}^\top}{k}\right|}{\left|\delta I_d + \frac{\psi\psi^\top}{k}\right|}\right)\right]\\
        & \overset{(d)}{=} \frac{1}{n}\E\left[\lim_{\delta\to 0 }\frac{1}{2}\log\left(\frac{\left|\delta I_d\right|\cdot\left|I_k + \frac{\tilde{\psi}^\top\tilde{\psi}}{k\delta}\right|}{\left|\delta I_d \right|\cdot\left|I_k + \frac{\psi^\top\psi}{k\delta}\right|}\right)\right]\\
        & = \frac{1}{n}\E\left[\lim_{\delta\to 0 }\frac{1}{2}\log\left(\frac{\left|I_k + \frac{\tilde{\psi}^\top\tilde{\psi}}{k\delta}\right|}{\left|I_k + \frac{I_k}{k\delta}\right|}\right)\right]\\
        & \overset{(e)}{\leq} \lim_{\delta\to 0 } \frac{1}{2n}\log\left(\frac{\left|I_k + \frac{\E\left[\tilde{\psi}^\top\tilde{\psi}\right]}{k\delta}\right|}{\left|I_k + \frac{I_k}{k\delta}\right|}\right)\\
        & = \lim_{\delta\to 0 } \frac{1}{2n}\log\left(\frac{\left|I_k + \frac{\E\left[I_k + d\epsilon I_k\right]}{k\delta}\right|}{\left|I_k + \frac{I_k}{k\delta}\right|}\right)\\
        & = \lim_{\delta\to 0 } \frac{k}{2n}\log\left(\frac{1 + \frac{1+d\epsilon}{k\delta}}{1+\frac{1}{k\delta}}\right)\\
        & = \frac{k}{2n}\log\left(1 + d\epsilon \right),
    \end{align*}
    where $(a)$, $(b)$ follows from continuity of the KL-divergence between two multivariate normal distributions w.r.t the covariance matrix, $(c)$ follows from the fact that the trace term is upper bounded by $d$, $(d)$ follows from the matrix determinant lemma, 
    $\epsilon = \frac{1}{m}$, and $(e)$ follows from Jensen's inequality.
\end{proof}

\begin{lemma}{\bf (distortion upper bound)}\label{le:distortion_ub2}
    For all $n, r \in \mathbb{Z}_{++}$,
    $$\frac{\I(H_{1,n};\theta_1|\psi)}{n}\ \leq\ \inf_{\epsilon\geq 0}\ \frac{r\log\left(1 + \frac{1}{r\epsilon}\right)}{2n} + r\epsilon$$
\end{lemma}
\begin{proof}
    Let $\tilde{\xi} = \xi + Z$ where $Z \perp \xi$ and $Z \sim \normal(0, \epsilon I_r)$.
    \begin{align*}
        \frac{\I(H_{1,n};\theta_1|\psi)}{n}
        & \overset{(a)}{=} \frac{\I(H_{1,n};\theta_1,\tilde{\xi}|\psi)}{n}\\
        & = \frac{\I(H_{1,n};\tilde{\xi}|\psi) + \I(H_{1,n};\theta_1|\tilde{\xi}, \psi)}{n}\\
        & \overset{(b)}{=} \frac{\I(H_{1,n};\tilde{\xi}|\psi) + \sum_{j=1}^{n}\I(Y_{j}^{(1)};\theta_1|\tilde{\xi}, \psi, H_{1,j-1}, X_j^{(1)})}{n}\\
        & = \frac{\I(H_{1,n};\tilde{\xi}|\psi) + \sum_{j=1}^{n}\H(Y_{j}^{(1)}|\tilde{\xi}, \psi, H_{1,j-1}, X_j^{(1)}) - \H(Y_j^{(1)}|\theta_1,\psi,\tilde{\xi}, H_{1,j-1}, X_j^{(1)})}{n}\\
        & \overset{(c)}{=} \frac{\I(H_{1,n};\tilde{\xi}|\psi) + \sum_{j=1}^{n}\H(Y_{j}^{(1)}|\tilde{\xi}, \psi, H_{1,j-1}, X_j^{(1)}) - \H(Y_j^{(1)}|\theta_1,\psi,\tilde{\xi},X_j^{(1)})}{n}\\
        & \overset{(d)}{\leq} \frac{\I(H_{1,n};\tilde{\xi}|\psi) + \sum_{j=1}^{n}\H(Y_{j}^{(1)}|\tilde{\xi},\psi, X_j^{(1)}) - \H(Y_j^{(1)}|\theta_1,\psi,\tilde{\xi},X_j^{(1)})}{n}\\
        & = \frac{\I(H_{1,n};\tilde{\xi}|\psi)}{n} + \I(Y_{j}^{(1)};\theta_1|\tilde{\xi},\psi,X_1^{(1)})\\
        & \leq \frac{\I(\xi;\tilde{\xi}|\psi)}{n} + \I(Y_{j}^{(1)};\theta_1|\tilde{\xi},\psi),
    \end{align*}
    where $(a)$ follows from the fact that $H_{1,n}\perp\tilde{\xi}|\psi,\theta_1$, $(b)$ follows from the chain rule of mutual information, $(c)$ follows from the fact that $(X_j^{(1)}, Y_j^{(1)})$ is iid $|\theta_1$, $(d)$ follows from the fact that conditioning reduces differential entropy, and $(e)$ follows from the data processing inequality applied to the markov chain $H_{1,n}\perp\tilde{\xi}|(\xi,\psi)$.

    We now upper bound the two above terms.
    \begin{align*}
        \frac{\I(\xi;\tilde{\xi}|\psi)}{n}
        & = \frac{\diffentropy(\tilde{\xi}|\psi)- \diffentropy(\tilde{\xi}|\psi,\xi)}{n}\\
        & = \frac{\diffentropy(\tilde{\xi})- \diffentropy(\tilde{\xi}|\xi)}{n}\\
        & = \frac{\diffentropy(\tilde{\xi})- \diffentropy(Z)}{n}\\
        & = \frac{\frac{r}{2}\log\left(2\pi e (\epsilon + \frac{1}{r})\right) - \frac{r}{2}\log\left(2\pi e\epsilon\right)}{n}\\
        & = \frac{r\log\left(1 + \frac{1}{r\epsilon}\right)}{2n}.
    \end{align*}

    Let $\tilde{\theta} = \psi\tilde{\xi}$. Then,
    \begin{align*}
        \I(Y_{j}^{(1)};\theta_1|\tilde{\xi},\psi)
        &\leq \I(Y_{j}^{(1)};\theta_1|\tilde{\theta})\\
        & =\E\left[\KL\left(\Pr\left(Y^{(1)}_j\in\cdot|\theta_1\right)\|\Pr\left(Y^{(1)}_{j}\in\cdot|\tilde{\theta}\right)\right)\right]\\
        & \overset{(a)}{\leq} \E\left[\KL\left(\Pr\left(Y^{(1)}_j\in\cdot|\theta_1\right)\|\Pr\left(Y^{(1)}_{j}\in\cdot|\theta_1\leftarrow\tilde{\theta}\right)\right)\right]\\
        & \overset{(b}{\leq} \E\left[\|\tilde{\theta}-\theta_1\|^2_2\right]\\
        & = \E\left[\left(\xi-\tilde{\xi}\right)^\top\psi^\top\psi\left(\xi-\tilde{\xi}\right)\right]\\
        & = \E\left[\left(\xi-\tilde{\xi}\right)^\top \left(\xi-\tilde{\xi}\right)\right]\\
        & = \E\left[Z^\top Z\right]\\
        & = r\epsilon
    \end{align*}
    where $(a)$ follows from Lemma \ref{le:bayes_optimal_seq}, and $(b)$ follows from Lemma \ref{le:sq_kl}.
\end{proof}

\linRepLearningBound*
\begin{proof}
    \begin{align*}
        \Lc_{m,n}
        & \overset{(a)}{\leq} \inf_{\epsilon \geq 0}\ \frac{dr\log\left(1 + \frac{1}{r\epsilon}\right)}{2mn} + \frac{r\log(1+d\epsilon)}{2n} + \inf_{\epsilon'\geq 0} \frac{r\log\left(1 + \frac{1}{r\epsilon'}\right)}{2n} + r\epsilon'\\
        & \overset{(b)}{\leq} \frac{dr\log\left(1 + \frac{m}{r}\right)}{2mn} + \frac{r\log\left(1 + \frac{d}{m}\right)}{2n} + \frac{r\log(1+\frac{2n}{r})}{2n} + \frac{r}{2n}\\
        & \leq \frac{dr\log\left(1 + \frac{m}{r}\right)}{2mn} + \frac{dr}{2mn} + \frac{r\log(1+\frac{2n}{r})}{2n} + \frac{r}{2n},
    \end{align*}
where $(a)$ follows directly from Lemmas \ref{le:rate_ub} and \ref{le:distortion_ub2}, and $(b)$ follows from setting $\epsilon = \frac{1}{m}$ and $\epsilon' = \frac{1}{2n}$. We choose these values because they are analytically simpler than the optimal values of $\epsilon, \epsilon'$ but are asymptotically identical to these optimal values.
\end{proof}

\subsection{Mixture of Transformer}
\begin{lemma}{\bf (sparse mixture meta-estimation error)}\label{le:sparse_meta_error}
    For all $R,M,T\in \Z_{++}$, 
    $$\I(H_{M,T};\psi) \leq R\log\left(1 + \frac{M}{R}\right)\log(MN).$$
\end{lemma}
\begin{proof}
    Recall that $\theta_{1:M}$ is distributed $\text{Dirichlet-Multinomial}(M, [R/N, \ldots, R/N])$.  Consider the following prefix-free coding scheme for $\theta_{1:M}$: For every nonzero category, allocate $\log(M)$ bits to designate the number of times that category was selected in $\theta_{1:M}$ with and an additional $\log(N)$ bits to designate the category $(1, \ldots, N)$.  We concatenate the bit strings for each such nonzero category.  As a result:
    \begin{align*}
        \I(H_{M,T};\psi)
        & \overset{(a)}{\leq} \I(\theta_{1:M};\psi)\\
        & \leq \H(\theta_{1:M})\\
        & \overset{(b)}{\leq} \E\left[\sum_{i=1}^{N} \mathbbm{1}_{[i\in \theta_{1:M}]}\left(\log(M) + \log(N)\right)\right]\\
        & \overset{(c)}{\leq} R\log\left(1 + \frac{M}{R}\right)\log(MN),
    \end{align*}
    where $(a)$ follows from the data processing inequality, $(b)$ follows from the fact that entropy is the minimum average prefix-free code length, and $(c)$ follows from the fact that the average number of non-zero outomes for a $\text{Dirichlet-Multinomial}(M, [R/N, \ldots, R/N])$ random variable is upper bounded by $R\log(1+M/R)$.
\end{proof}

\metaTransformer*
\begin{proof}
    Let $\tilde{\Theta}_N = \{\theta + Z_{\theta} : \theta \in \Theta\}$.  $\Theta$ is the set of $N$ transformer model weights for each of the $N$ models in the mixture and $Z_{\theta}\perp \theta$ is random noise of the following characteristic: $\theta = (A_{1:L}, V_{1:L}), \tilde{\theta} = (\tilde{A}_{1:L}, \tilde{V}_{1:L}) $, $Z_{\theta} = (Z^{\theta, A}_{1:L}, Z^{\theta, V}_{1:L})$, for all $i$, $\tilde{A}_{i} = A_i + Z^{\theta, A}_{i}, \tilde{V}_i = V_i + Z^{\theta, V}_i$ where $Z^{\theta, A}_{i}$ consists of elements drawn iid $\normal(0, \frac{2\epsilon T}{r(d^2+r^2)L^2\log(4K^2)})$ and $Z^{\theta, A}_i$ consists of elements drawn iid $\normal(0, \frac{2\epsilon T}{d^2(d^2+r^2)L^2\log(4K^2)})$.  $\tilde{\Theta}_N$ hence is a collection of lossy compressions of the models in the mixture.

    Let $\tilde{B} \in \{1, \ldots, N\}^M$ be the collection containing the outcomes which model from the mixture was ascribed to $\theta_{1}, \ldots, \theta_{M}$.  Since there are $N$ different transformers in the mixture, $\tilde{B}$ takes values in the set $\{1, \ldots, N\}^M$.
    
    \begin{align*}
        & \I(H_{M,T};\psi, \theta_{1:M})\\
        & = \I(H_{M,T};\psi, \theta_{1:M}, \tilde{\Theta}_N, \tilde{\beta})\\
        & = \I(H_{M,T};\psi) + \I(H_{M,T};\tilde{\Theta}_N, \tilde{\beta}|\psi) + \I(H_{M,T};\theta_{1:M}|\psi, \tilde{\Theta}_N, \tilde{\beta})\\
        & \overset{(a)}{\leq} \I(H_{M,T};\psi) + \I(\theta_{1:M};\tilde{\Theta}_N, \tilde{\beta}|\psi) + \I(H_{M,T};\theta_{1:M}|\psi, \tilde{\Theta}_N, \tilde{\beta})\\
        & \overset{(b)}{=} \I(H_{M,T};\psi) + \I(\theta_{1:M};\tilde{\beta}|\psi) + \I(\theta_{1:M};\tilde{\Theta}_N|\tilde{\beta}, \psi) + \I(H_{M,T};\theta_{1:M}|\psi, \tilde{\Theta}_N, \tilde{\beta})\\
        & \leq  \I(H_{M,T};\psi) + M\log(N) + \E\left[\E\left[\sum_{i=1}^{N} \mathbbm{1}_{[i\in\tilde{\beta}]} \cdot \I(\Theta[i];\tilde{\Theta}[i])\Big|\tilde{\beta}\right]\right] + \sum_{m=1}^{M}\sum_{t=0}^{T-1}\I(X_{t+1}^{(m)};\theta_{1:M}|\tilde{\Theta}_N, \tilde{\beta}, H_{m,t})\\
        & \overset{(c)}{\leq} \I(H_{M,T};\psi) + R\log\left(1+\frac{M}{R}\right)
        \left[\frac{(d^2+r^2)L^2\log\left(4K^2\right)}{2} + \frac{(d^2+r^2)L\log\left(\frac{2\max\{d,r\} dKLT}{\epsilon}\right)}{2}\right]\\
        &\quad + MT \epsilon + M\log(N)\\
        & \overset{(d)}{\leq} \I(H_{M,T};\psi) + R\log\left(1+\frac{M}{R}\right)
        \left[(d^2+r^2)L^2\log\left(4K^2\right) + \frac{(d^2+r^2)L\log\left(\frac{4KMT^2}{L}\right)}{2}\right] + M\log(N)\\
        & \leq R\log\left(1+\frac{M}{R}\right)
        \left[\log(MN) + (d^2+r^2)L^2\log\left(4K^2\right) +  \frac{(d^2+r^2)L\log\left(\frac{4KMT^2}{L}\right)}{2}\right] + M\log(N)
    \end{align*}
    where $(a)$ follows from the data processing inequality, $(b)$ follows from the chain rule of mutual information, $(c)$ follows from Theorem \ref{th:tsfmRd}, $(d)$ follows by setting $\epsilon = (d^2+r^2)L^2\log(4K^2)/2MT$, and $(e)$ follows from Lemma \ref{le:sparse_meta_error}.
\end{proof}

\subsection{In-context Learning}\label{apdx:icl}
\icl*
\begin{proof}
    \begin{align*}
        \L_{M,T,\tau}
        & = \frac{1}{\tau} \sum_{t=0}^{\tau-1} \E\left[-\log \Pr(X^{(M+1)}_{t+1}|H_{M+1, t})\right]\\
        & = \frac{1}{\tau} \sum_{t=0}^{\tau-1} \E\left[\log \frac{1}{\Pr(X^{(M+1)}_{t+1}|\theta_{M+1}, H_{M+1,t})} + \KL(\Pr(X^{(M+1)}_{t+1}\in\cdot|\theta_{M+1}, H_{M+1, t}) \| \Pr(X^{(M+1)}_{t+1}\in\cdot|H_{M+1, t}))\right]\\
        & = \frac{1}{\tau} \sum_{t=0}^{\tau-1} \H(X_{t+1}^{(M+1)}|\theta_{M+1}, H_{M+1, t}) + \I(X_{t+1}^{(M+1)};\theta_{M+1}|H_{M+1, t})\\
        & = \frac{1}{\tau} \sum_{t=0}^{\tau-1} \H(X_{t+1}^{(M+1)}|\theta_{M+1}, X^{(M+1)}_{1},\ldots, X^{(M+1)}_{t}) + \I(X_{t+1}^{(M+1)};\theta_{M+1}, \psi|H_{M+1, t})\\
        & \overset{(a)}{=} \frac{\H(D_{M+1}|\theta_{M+1})}{\tau} + \frac{\I(H_{M+1,\tau};\theta_{M+1}, \psi|H_{M+1,0})}{\tau}\\
        & \overset{(b)}{=} \frac{\H(D_{M+1}|\theta_{M+1})}{\tau} + \frac{\I(H_{M+1,\tau};\psi|H_{M+1,0})}{\tau} + \frac{\I(H_{M+1,\tau};\theta_{M+1}|\psi, H_{M+1,0})}{\tau}\\
        & \overset{(c)}{\leq} \frac{\H(D_{M+1}|\theta_{M+1})}{\tau} + \frac{\I(H_{M+1,T};\psi|H_{M+1,0})}{\tau} + \frac{\I(D_{M+1};\theta_{M+1}|\psi)}{\tau}\\
        & \overset{(d)}{\leq} \frac{\H(D_{M+1}|\theta_{M+1})}{\tau} + \frac{\I(H_{M+1,T};\psi)}{(M+1)\tau} + \frac{\I(D_{M+1};\theta_{M+1}|\psi)}{\tau},
    \end{align*}
    where $(a)$ and $(b)$ follow from the chain rule of mutual information, $(c)$ follows from the fact that $\psi \perp H_{M+1,\tau}|H_{M+1,T}$ for $\tau \leq T$ and the data processing inequality, and $(d)$ follows from the fact that for all $m$, $\I(H_{m+1,T};\psi|H_{m,T}) \leq \I(H_{m,T};\psi|H_{m-1,T})$ and the chain rule of mutual information.
\end{proof}

\section{Analysis of Suboptimal Meta-Learning Algorithms}\label{apdx:suboptimal}
All of the prior results bound the error incurred by the optimal algorithm which produces a prediction of the next token conditioned on the entire past sequence.  In this section, we will derive some simple results which pertain to \emph{suboptimal} algorithms.

The following result quantifies the shortfall incurred by an algorithm which produces an arbitrary prediction $\tilde{P}_{m,t}$ which may depend on the history $H_{m,t}$.
\begin{lemma}{\bf (loss of an arbitrary predictor)}\label{le:lossQ}
    For all $M,T \in \mathbb{Z}_{++}$, if for all $(m,t) \in [M]\times[T]$, $\tilde{P}_{m,t}$ is a predictive distribution which may depend on the previous data $H_{m,t}$ and $\tilde{\Lc}_{M,T}$ denotes its cumulative average log-loss, then
    $$\tilde{\L}_{M,T} = \L_{m,n} + \underbrace{\frac{1}{MT}\sum_{m=1}^{M}\sum_{t=0}^{T-1} \E\left[\KL\left(\hat{P}_{m,t}\ \big\|\ \tilde{P}_{m,t}\right)\right]}_{\rm misspecification\ error},$$
\end{lemma}
 Note that because KL divergence is always non-negative and $\Lc_{m,n}$ is the loss of the Bayesian posterior estimator $\hat{P}$, any prediction other than $\hat{P}$ will incur nonzero misspecification error.

 For a particular class of predictors $\tilde{P}$, we can retrieve the following upper bound on the \emph{misspecification error}. We consider predictors which perform Bayesian inference with respect to an incorrectly specified prior distribution $\tilde{P}_0$.

\begin{restatable}{theorem}{misspecified}{\bf (misspecified prior error bound)}\label{th:prior_ub}
    For all $M,T\in\mathbb{Z}_{++}$ and $m,t\in [M]\times[T]$,
    if $\tilde{P}_{m,t}$ is the Bayesian posterior under the prior $\tilde{P}_0(\psi)$, then
    \begin{align*}
        \frac{1}{MT}\sum_{m=1}^{M}\sum_{t=0}^{T-1} \E\left[\KL\left(\hat{P}_{m,t}\ \big\|\ \tilde{P}_{m,t}\right)\right] \leq \frac{\E\left[\KL\left(\Pr(\psi\in\cdot)\|\tilde{P}_0(\psi\in\cdot)\right)\right]}{MT}.
    \end{align*}
\end{restatable}
\begin{proof}
    \begin{align*}
        & \frac{1}{MT}\sum_{m=1}^{M}\sum_{t=0}^{T-1} \E\left[\KL\left(\hat{P}_{m,t}\ \big\|\ \tilde{P}_{m,t}\right)\right]\\
        & \overset{(a)}{=} \frac{1}{MT}\sum_{m=1}^{M} \E\left[\KL\left(\Pr(H_{T}^{(m)}\in\cdot)\ \big\|\ \tilde{P}_{m}\left(H_{T}^{(m)}\in\cdot\right)\right)\right]\\
        & \overset{(b)}{=} \frac{\E\left[\KL(\Pr(H_{M,T} \in \cdot)\ \|\ \tilde{P}(H_{M,T}\in\cdot))\right]}{MT}\\
        & \overset{(c)}{\leq} \frac{\E\left[\KL\left(\Pr(\psi\in\cdot)\ \|\ \tilde{P}_0(\psi\in\cdot)\right)\right]}{MT},
    \end{align*}
    where $(a)$ and $(b)$ follow from the chain rule of KL divergence and $(c)$ follows from the data processing inequality of KL Divergence.
\end{proof}

Theorem \ref{th:prior_ub} suggests that so long as the KL divergence between prior distributions is finite, the misspecification error should decrease to $0$ as $M$ and $T$ $\to\infty$. This can be ensured so long as the algorithm's prior $\tilde{P}_0(\psi)$ does not assign $0$ probability mass to any set for which the environment prior $\Pr(\psi)$ assigns non-zero probability.

With these results in place, we provide the following Corollary which exactly characterizes the loss of a predictor $\tilde{P}$ which produces predictions via Bayesian inference with respect to a arbitrary prior distribution $\tilde{P}_0(\psi\in\cdot)$.
\begin{corollary}
    For all $M,T\in\mathbb{Z}_{++}$ and $m,t\in [M]\times[T]$,
    if $\tilde{P}_{m,t}$ computes probabilities under an arbitrary prior distribution $\tilde{P}_0(\psi\in\cdot)$ and $\tilde{\Lc}_{M,T}$ denotes its cumulative average log-loss,, then
    \begin{align*}
        \tilde{\L}_{M,T}
        & = \frac{\H(H_{M,T}|\theta_{1:M})}{MT} + \frac{\I\left(H_{M,T};\psi\right)}{MT} + \frac{\I\left(D_m;\theta_m|\psi\right)}{T}\\
        & + \frac{\E\left[\KL\left(\Pr\left(H_{M,T}\in\cdot\right)\|\tilde{P}\left(H_{M,T}\in\cdot\right)\right)\right]}{MT}.
    \end{align*}
\end{corollary}


\end{document}